\documentclass[letter]{article}
\usepackage{amsmath,amsthm,mathrsfs}
\usepackage{mathtools}
\mathtoolsset{showonlyrefs}

\usepackage{natbib}

\usepackage{tabulary}
\usepackage{array,hhline}

\usepackage{amssymb}
\usepackage[utf8]{inputenc}
\usepackage{geometry}

\usepackage[T1]{fontenc}
\usepackage{mathpazo,tgcursor,tgtermes}
\usepackage[scale=.9]{tgheros}
\usepackage{textcomp}

\usepackage[footnotesize,bf]{caption}

\usepackage{graphicx,authblk}

\usepackage{color}

\usepackage{algorithm}
\usepackage{algorithmic}

\renewcommand{\algorithmicrequire}{\textbf{Input:}}
\renewcommand{\algorithmicensure}{\textbf{Output:}}

\usepackage{xspace}
\usepackage{enumitem}
\setenumerate[1]{label=(\roman*)}
\usepackage{xparse}

\usepackage[hyphens]{url}
\usepackage[bookmarksnumbered,bookmarksopen,unicode]{hyperref}
\usepackage{doi}

\newlist{enumerate*}{enumerate*}{1}
\setlist[enumerate*]{label=(\arabic*),
  itemjoin={{, }}, itemjoin*={{, and }}, after={.}}

\newcommand{\R}{\mathbb{R}}
\newcommand{\N}{\mathbb{N}}

\newcommand{\bp}{\mathbf{p}}
\newcommand{\br}{\mathbf{r}}

\DeclareMathOperator*{\argmax}{arg\,max}

\NewDocumentCommand{\conv}{mo}{\operatorname{conv}\left(#1%
    \IfValueT{#2}{\,\middle|\,#2}%
  \right)}
\NewDocumentCommand{\cone}{mo}{\operatorname{cone}\left(#1%
    \IfValueT{#2}{\,\middle|\,#2}%
  \right)}

\DeclarePairedDelimiterX\inp[2]{\langle}{\rangle}{#1,#2}

\NewDocumentCommand{\probability}{d()om}{%
  \operatorname{\mathbb{P}}%
  \IfValueT{#1}{\sb{#1}}%
  \left[#3%
    \IfValueT{#2}{\,\middle|\,#2}\right]}
\NewDocumentCommand{\expectation}{d()om}%
  {\operatorname{\mathbb{E}}%
    \IfValueT{#1}{\sb{#1}}%
    \left[#3%
      \IfValueT{#2}{\,\middle|\,#2}\right]}
\NewDocumentCommand{\variance}{d()om}%
  {\operatorname{Var}%
    \IfValueT{#1}{\sb{#1}}%
    \left[#3%
      \IfValueT{#2}{\,\middle|\,#2}\right]}

\NewDocumentCommand{\norm}{osm}{%
  \IfBooleanT{#2}{\left}\lVert
    #3%
    \IfBooleanT{#2}{\right}\rVert
  \IfValueT{#1}{\sb{#1}}}

\newcommand*{\size}[1]{\left|#1\right|}

\newcommand{\mbx}{\mathbf{x}}
\newcommand{\mbu}{\mathbf{u}}
\newcommand{\mbe}{\mathbf{e}}
\newcommand{\mby}{\mathbf{y}}
\newcommand{\blambda}{\boldsymbol{\lambda}}

\newtheorem{theorem}{Theorem}[section]

\newtheorem{proposition}[theorem]{Proposition}
\newtheorem{corollary}[theorem]{Corollary}

\theoremstyle{remark}
\newtheorem{remark}[theorem]{Remark}
\newtheorem{example}[theorem]{Example}
\theoremstyle{definition}
\newtheorem{definition}[theorem]{Definition}
\newtheorem{assumption}[theorem]{Assumption}

\let\citet\cite

\hypersetup{
  pdftitle={Robust Optimization via Online Learning Revisited},
  pdfauthor={Sebastian Pokutta, Huan Xu},
  pdfkeywords={regret minimization, bandit feedback,
    implicit exploration},
  pdfsubject={MSC2000 Primary; Secondary}
}

\title{Adversaries in Online Learning Revisited:\\
with applications in Robust Optimization\\ and Adversarial training}
\date{January 27, 2021 \\ (first version: Feburary 12, 2017)}

\author[1]{Sebastian Pokutta}
\affil[1]{Zuse Institute Berlin and Technische Universität Berlin, Germany.
  \textit{Email:}~pokutta@zib.de}

\author[2]{Huan Xu}
\affil[2]{Alibaba Inc, USA.
  \textit{Email:}~huan.xu@alibaba-inc.com}

\begin{document}
\maketitle

\begin{abstract}
  We revisit the concept of ``adversary'' in online learning,
  motivated by solving robust optimization and adversarial training
  using online learning methods.  While one of the classical setups in
  online learning deals with the ``adversarial'' setup, it appears that
  this concept is used less rigorously, causing confusion in
  applying results and insights from online learning.
  Specifically, there are two fundamentally different types of
  adversaries, depending on whether the ``adversary'' is able to
  anticipate the exogenous randomness of the online learning
  algorithms. This is particularly relevant to robust optimization and
  adversarial training because the adversarial sequences are often
  anticipative, and many online learning algorithms do not achieve
  diminishing regret in such a case.
  
  We then apply this to solving robust optimization problems or
  (equivalently) adversarial training problems via online learning and
  establish a general approach for a large variety of problem classes
  using \emph{imaginary play}. Here two players play against each
  other, the primal player playing the decisions and the dual player
  playing realizations of uncertain data. When the game terminates,
  the primal player has obtained an approximately robust solution.
  This meta-game allows for solving a large variety of robust
  optimization and multi-objective optimization problems and
  generalizes the approach of \citet{ben2015oracle}.

\end{abstract}

\section{Introduction}
\label{sec:introduction}
This paper, motivated by solving robust optimization and
(equivalently) adversarial training, revisits the concept of
``adversary'' in online learning. A significant amount of literature
in online learning focuses on the so-called ``adversarial'' setup, where a
learner picks a sequence of solutions against a sequence of inputs
``chosen adversarially'', and achieves accumulated utility almost as
good as the best fixed solution in hindsight. While those results are
widely known and applied, we observe that the concept of ``adversary''
is often understood and applied in an incorrect way, partly because of
a lack of a rigorous definition, which we address in this paper. Our
observation is largely motivated by recent works applying online
learning to solve robust optimization and adversarial training, where
diminishing regret, contrary to the claim, is not guaranteed to be achieved.

Robust optimization \citep[see e.g.,
][]{ben1998robust,BN00,ben2002robust,Bertsimas04,BGN09,bertsimas2011theory}
is a powerful approach to account for data uncertainty when
distributional information is not available.  Taking a worst-case
perspective in \emph{Robust Optimization (RO)} we are interested in
finding a good solution against the worst-case data realization, which
leads to problems of the form
\begin{equation}
  \label{eq:robOptProblemIntro}
\mbox{Minimize}_{\mbx\in \mathcal{X}}\,\, \mbox{Maximize}_{\mbu\in \mathcal{U}} \,\, f(\mbx, \mbu),
\end{equation}
where $\mathcal{X}$ is the decision set and $\mathcal{U}$ is the
uncertainty set and it is well known that robust optimization can be
expressed in this standard form via appropriate choice of $f$. We thus
look for a \emph{robust solution} $\mbx$ that minimizes the cost $f$
under the worst-case realization $\mbu$. In particular, when the
function \(f\)
is convex in the decisions \(\mbx\)
and concave in the uncertainty \(\mbu\)
the resulting problem is convex and there is a plethora of approaches
to solve these problems. In fact, robust optimization is tractable in
many cases and can be solved often via methods from convex
optimization as well as integer programming in the context of discrete
decision; we refer the interested reader to \citet{BGN09} for an
introduction and to, e.g., \citet{bertsimas2003robust} for robust
optimization with discrete decisions.

Robust optimization is also closely related to {\em adversarial
  training}, a subject that has drawn significant attention in machine
learning research, and particularly in deep learning.  It has been
observed that a well-trained deep neural network can often be
``fooled'' by small perturbations. That is, if a data point (an
image for example) is perturbed by a carefully designed adversarial
noise, potentially imperceivable to the human being (i.e., with a
small $\ell_{\infty}$ norm in the case of images) , the classification result can be
completely changed. To mitigate this issue, adversarial training aims
to train a neural network to be robust to such perturbations, which can be
formulated as follows:
$$\underset{\mathbf{w}}{\mbox{minimize}} \underset{(\delta_1, \delta_2,\cdots, \delta_n)}{\mbox{maximize}}  \frac{1}{n} \sum_{i=1}^n l(y_i, \phi(\mathbf{w}, x_i+\delta_i)),
$$
where $\mathbf{w}$ is the weight vector of the neural network,
$(x_i, y_i)$ are the $i$-th data point, $\phi(\mathbf{w}, x)$ is the
prediction of the neural network on an input $x$ and $l(\cdot, \cdot)$
is the loss function.  Thus, adversarial training essentially is an
attempt to solve an (often non-convex and hard) robust optimization
problem. Due to the lack of convexity, exactly solving the above
formulation is intractable, and numerous heuristics has been proposed
to address the computational issue.

Recently, several works explored a general framework to solve robust
optimization and adversarial training via online learning. The main
idea is the following: instead of solving the robust problem
one-shot by exploiting convex duality, a sequence of scenarios
$\mbu_i\in \mathcal{U}$ is generated using online learning, and
optimal decisions (or approximate solutions for very complicated
functions, in the adversarial training case) for each scenarios are
then averaged as the final output. Using theorems from online learning,
it is shown that the final output is close to optimal (or achieves
same approximation ratio, for adversarial training) for \emph{all}
scenarios in $\mathcal{U}$.

The framework outlined above can be very appealing
computationally. However, a close examination of the argument shows
that because of the ambiguity on the concept ``adversary'' in online
learning, some of the claimed results are invalid (see
Section~\ref{sub.counter} for a concrete counter-example). Thus, we
feel that it is necessary to characterize the concept ``adversary'' in
a more rigorous way to avoid future confusion. This also enables us
to develop new methods for solving robust optimization using online
learning.

\paragraph{Contribution.}
We now summarize our contribution:

\emph{Clarification of concepts.} The main contribution in this paper
is to distinguish two types of adversaries in an online learning
setup, which we termed ``anticipatory'' and ``non-anticipatory''. In a
nutshell, anticipatory adversaries are those that have access to
inherent randomness of the online learning algorithm {\em at the
  current step}. The example that motivates this concept is when the
adversary is chosen by solving an optimization problem whose
parameters are the output of the online learning
algorithm. Non-anticipatory adversaries do not have access to the
inherent randomness of the current step (however can still be
adaptive). Based on that, we further distinguish two types of online
learning algorithms, which are both known in the literature to achieve
diminishing regret with respect to adversarial input. Depending on
whether such adversarial input can be anticipatory or not, we call the
two classes strong learners and weak learners.

\emph{One-sided minimax problems via imaginary play.} We then apply
our model to the special case of solving robust optimization
problems. We show how to solve problems of the form
\eqref{eq:robOptProblemIntro} by means of online learning with two
weak learners. Slightly simplified, two learners play against each
other solving Problem~\eqref{eq:robOptProblemIntro}. However, in
contrast to general saddle point problems, only one of the players can
extract a feasible solution, as we considerably weaken convexity
requirements both for the domains as well as the functions $f$ (or
even drop them altogether). For this we present a general primal-dual
setup that is then later instantiated for specific applications by
means of plugging-in the desired learners.

\emph{Biased play with asymmetric learners.} We then show how to
further gain flexibility by allowing asymmetry between the
learners. Here one learner is weakened (in terms of requirements) to an
optimization oracle, and consequently the other player is strengthened
to allow anticipatory inputs.

\emph{Applications.} Finally, we demonstrate how our approach can be
used to solved a large variety of robust optimization problems. For
example, we show how to solve robust optimization problems with
complicated feasible sets involves integer programming. Another
example is robust MDPs with {\em non-rectangular} uncertainty sets,
where {\em only the reward parameters} are subject to uncertainty. Due
to space constraints, we defer the applications to the appendix.

\section{Preliminaries and motivation}
\label{sec:prelim}

In the following let \(\Delta(n)\) denote the unit simplex in
\(\R^n\). We will use the shorthand \([m]\) to denote the set
\(\{1, \dots, m\}\). For the sake of exhibition we will differentiate
between \emph{\(\mbox{Maximize}\)} and \emph{\(max\)}, where the
former indicates that we maximize a function via an algorithm, whereas
the latter is simply indicating the maximum function without any
algorithmic reference. Moreover, we will denote both the decision
vector \(\mbx\) as well as the uncertainty vector \(\mbu\) in bold
letters. All other notations are standard if not defined otherwise.

\subsection{Conventional wisdom}

In this work we consider games between two player and we will use
robust optimization or adversarial training of the form \eqref{eq:robOptProblemIntro}
as our running example. For the sake of continuity, we adapt the notation
of \cite{ben2015oracle}, however we stress that we later
will selectively relax some of the assumptions. Consider:
\begin{align*}
  \text{Minimize}_{\mbx} \ & \left\{f_0(\mbx)\,\,\right|\left.
  f_i(\mbx,\mbu_i) \leq 0 ,\,\, i \in [m];\quad
 \mbx \in \mathcal X\right\},
\end{align*}
where \(\mathcal X \subseteq \R^n\)
is the domain of feasible decisions and the \(f_i\)
with \(i \in [m]\)
are convex functions in \(\mbx\)
that are parametrized via vectors \(\mbu_i \in \R^d\)
for some \(d \in \N\) for \(i \in [m]\). The problem above is
parametrized by a fixed choice of vectors \(\mbu_i\) with \(i \in
[m]\) and we will refer to a problem in this form as the \emph{nominal
problem} (with parameters \(\{\mbu_i\}_i\)), which corresponds to the
outer minimization problem given a realization of the adversary's
choice \(\{\mbu_i\}_i\). 

In robust optimization we \emph{robustify} the nominal problem against
the worst-case choice of  \(\{\mbu_i\}_i\)   via the
formulation:
\begin{align*}
  \underset{\mbx}{\text{Minimize}} \ & \left\{f_0(\mbx)\,\,\right|\left.
  f_i(\mbx,\mbu_i) \leq 0 ,\,\forall \mbu_i \in \mathcal U_i,\,  i \in [m];
 \mbx \in \mathcal X\right\},
\end{align*}
where the \emph{uncertainty set} \(\mathcal U_i\)   is the set of possible choices of parameter $\mbu_i$. Thus, denote $\mathcal{U}= \prod_{i \in
  [m]} \mathcal U_i$ and we have \((\mbu_1, \dots, \mbu_m) \in
\mathcal U\).
 
Recently there has been a line of work proposing methods to solve
adversarial training and robust optimization via online learning
methods or in an equivalent fashion (see e.g.,
\cite{ben2015oracle,madry2017towards,chen2017robust,sinha2017certifiable}). In
all cases the underlying meta-algorithm works as follows: the
\(\mbu\)-player takes $\{\mbx\}$ as input and generate a sequence of
$\{\mbu\}$ according to an online learning algorithm which achieves
diminishing regret against {\em adversarial input}. The
\(\mbx\)-player on the other hand, computes $\mbx_t$ by minimizing the
loss function with $\mbu_t$ as input; the interpretation of the roles
of the players depends on the considered problem. 

In particular, in \cite{ben2015oracle}, the authors proposed two
methods along this line, using {\em online convex optimization} and
{\em Follow the Perturbed Leader (FPL)} as the online learning
algorithm, respectively. In \cite{chen2017robust}, the authors
consider the case where $\mathcal{U}$ is a finite set, and proposed to
use {\em exponential weighting} as the online learning algorithm (in
the infinite case, they use online gradient descent), and then solve
$\mbx_t$ by minimizing the loss function for the distributional
problem. While superficially similar, these two approaches are
markedly different as we will see.
 
\subsection{A motivating counter example}\label{sub.counter}
Unfortunately, the outlined approach above can be easily flawed, for reasons that will be
made clear later. We start with the following counter example, and apply
the second method (i.e., FPL based one) proposed in
\cite{ben2015oracle}.

Consider the following robust feasibility problem: Let
$\mathcal{U}=conv\{(1,1),(-1, -1), (2,1)\}$, and $\mathcal{X} =conv\{(0,1), (-2, -1), (0,0)\}$
does there exist $\mathbf{x} \in\mathcal{X} $
such that
$$\min_{\mbu\in \mathcal{U}} \mbu^\top \mbx \geq 1?
$$
The answer is clearly negative, as for any $\mbx$, at least one of
$(1,1)^\top \mbx$ and $(-1,-1)^\top \mbx$ is less than or equal to
\(0\). However, Theorem~2 of \cite{ben2015oracle} asserts that if we
update $\mbu_t$ according to \emph{Follow the Perturbed Leader},
compute $\mbx_t$ via
$$\mbx_t:=\argmax_{\mbx\in \mathcal{X}}  \mbu_t^\top \mbx,
$$
and if $\mbu_t^\top \mbx_t \geq 1$ for $t=1,2,\cdots,T$ (which is true
here as the objectives are linear and thus \(\mbu_t\) is a vertex of
\(\mathcal U\)),  then
$\overline\mbx=\sum_{t=1}^T \mbx_t \in \mathcal{X}$ is $\epsilon$-robust feasible,
i.e.,
$$\min_{u\in \mathcal{U}} \mbu^\top \overline{\mbx} \geq 1-\epsilon,
$$
where $\epsilon=O(1/\sqrt{T})$; this is clearly not true and we obtain
the desired contradiction.

Furthermore, we also show that $\overline\mbx$ does not converge to
the minimax solution: Since $\mbu_t$ is obtained by FPL and the
objective is a linear function as before we have that $\mbu_t$ is a
vertex of $\mathcal{U}$. Further notice that for
$\mbu_t \in \{(1,1),(-1, -1), (2,1)\}$ we have
$\mbx_t=\argmax_{\mbx\in \mathcal{X}} \mbu_t^\top \mbx \in \{(0,1),
(-2, -1)\}$. Therefore $\overline\mbx$ is on the line segment
between $(0,1)$ and $(-2, -1)$. One can easily check that for any
$\mbx$ on this line segment, we have
$$\min_{\mbu\in \mathcal{U}}  \mbu^\top \mbx \leq -1/5.
$$
On the other hand, clearly for $\mbx^*=(0,0)$, we have
$$\min_{\mbu\in \mathcal{U}}  \mbu^\top \mbx^* =0.
$$
Thus, $\overline\mbx$ does not converge to the minimax solution.

Interestingly, the first method proposed in \cite{ben2015oracle} turns
out to be a valid method for this example. Also, the approach in
\cite{chen2017robust} does not suffer from this weakness as the
Bayesian optimization oracle is applied to the \emph{output
  distribution}, rather than a sampled solution (which would be
problematic). Indeed, this is no coincidence. To clearly explain the
different behaviors for various methods proposed in literature is the
main motivation of this work.

\section{Anticipatory and non-anticipatory adversaries}
In this section we provide definitions of the main concepts that we
are introducing in this paper. There are two types of ``adversaries'',
namely \emph{anticipatory adversaries} and \emph{non-anticipatory
  adversaries} for online learning setups that need to be clearly
distinguished. In a nutshell, slightly simplifying, the distinction is whether the
adversary's decision (who is potentially computationally unbounded) is
independent of the private randomness \(\xi_t\) of the current round \(t\); if
not the adversary might be able to anticipate the player's decision \(\mbx_t\). 

\begin{definition}
  An {\em online learning setup} is as follows: for $t=1,2,\cdots$,
  the algorithm is given access to an external signal $\mby_t$, and an
  exogenous random variable $\xi_t$ which are independent with
  everything else, and furthermore $\xi_i$ and $\xi_j$ are independent
  for $i\not=j$. An {\em online learning algorithm} is a mapping:
$$  \mbx_t := \mathcal{L}_{x}(\mbx_1,\mby_1,\mbx_2, \mby_2,\cdots, \mbx_{t-1},
  \mby_{t-1},\xi_t).
$$
We say an online learning algorithm is {\em deterministic} if it is a mapping 
$$  \mbx_t := \mathcal{L}_{x}(\mbx_1,\mby_1,\mbx_2, \mby_2,\cdots, \mbx_{t-1},
  \mby_{t-1}).
$$
\end{definition}
In other words, an online learning algorithm picks an action at time
$t$ depending on past actions $\mbx_1,\cdots,\mbx_{t-1}$, past signals
$\mby_1,\cdots, \mby_{t-1}$, and an exogenous randomness $\xi_t$. And
a deterministic online learning algorithm is independent of the
exogenous randomness.

Existing analyses for online learning algorithms focus on two cases:
either the external signal $\mby_t$ is generated stochastically (and
typically in an iid fashion), or it is generated
``adversarially''. However, as we will show later, the term
``adversarially'' is loosely defined and causes significant
confusion. Instead, we now define two types of adversary signals.

\begin{definition}
Recall an online learning algorithm is given access
  to exogenous random variables $\xi_1, \xi_2,\cdots, \xi_t,\cdots,$
  and its output $\mbx_t$ may depend on $\xi_1,\cdots \xi_t$, but is
  independent to $\xi_{t+1}, \xi_{t+2}, \cdots$. A sequence
  $\mby_1,\mby_2,\cdots, \mby_t, \cdots$ is called {\em non-oblivious
    non-anticipatory} (NONA) with respect to $\{\mbx\}$ if $\mby_t$
  may depend on $\xi_1, \cdots, \xi_{t-1}$, but is independent of 
  $\xi_t, \xi_{t+1},\cdots$, for all $t$. A sequence
  $\mby_1,\mby_2,\cdots, \mby_t, \cdots$ is called {\em anticipatory}
  wrt $\{\mbx\}$ if $\mby_t$ may depend on $\xi_1, \cdots, \xi_{t}$,
  but is independent of $\xi_{t+1}, \xi_{t+2},\cdots$, for all $t$.
\end{definition}

We now provide some examples to illustrate the concept.

\begin{enumerate}
\item If $\{\mby\}$ is a sequence chosen arbitrarily, independent of
  $\{\xi\}$, then it is a  NONA sequence.
\item If $\mby_t$ is chosen according to
$$\mby_t= \mathcal{F}_t(\mby_1, \mbx_1,\cdots, \mby_{t-1}, \mbx_{t-1}),
$$
for some function $\mathcal{F}_t(\cdot)$, then $\{\mby\}$  is a NONA sequence. 
\item If $\mby_t$ is chosen according to
$$\mby_t= \mathcal{F}_t(\mbx_1, \mby_1,\cdots, \mbx_{t-1}, \mby_{t-1}, \mbx_t),
$$
for some function $\mathcal{F}_t(\cdot)$, then $\{\mby\}$  is an
anticipatory sequence. This is because $\mbx_t$ is (potentially) dependent to $\xi_t$, and so is $\mby_t$. As a special case, suppose $$\mby_t= \arg\max_{\mby\in \mathcal{Y}} f(\mbx_t, \mby),$$ then  $\{\mby\}$  is an anticipatory sequence.
\item If $\{\mbx\}$ is the output of a deterministic online learning algorithm, and $\mby_t$ is chosen according to
$$\mby_t= \mathcal{F}_t(\mbx_1, \mby_1,\cdots, \mbx_{t-1}, \mby_{t-1}, \mbx_t),
$$
for some function $\mathcal{F}_t(\cdot)$, then $\{\mby\}$  is a NONA sequence. This is because $\mbx_t$ is independent of $\xi_t$ since the online learning algorithm is deterministic. In this case, the following sequence is NONA as well:  
$$\mby_t= \arg\max_{\mby\in \mathcal{Y}} f(\mbx_t, \mby).$$ 
\end{enumerate}

The standard target of online learning algorithms is to achieve diminishing regret {\em vis a vis} a sequence of external signals. Thus, depending on whether the external signal is anticipatory or not, we define two class of learning algorithms.
\begin{definition}
Suppose $\mathcal{X}$ is the feasible set of actions, and for $\mbx_t$ the action chosen at time $t$, it is evaluated by $f(\mbx_t, \mby_t)$, with a smaller value being more desirable. \begin{enumerate}
\item We call an online learning algorithm $\mathcal{L}_{x}$ for
  $\mbx$ a {\em weak learning algorithm} with regret $R(\cdot, \cdot)$,  if for any NONA sequence $\{\mby_t\}_{t=1}^\infty$,
  the following holds with a probability $1-\delta$ (over the exogenous randomness of the algorithm), where $\{\mbx_t\}$ are the output of $\mathcal{L}_{x}$:
$$ \sum_{t=1}^T f(\mbx_t, \mby_t) -\min_{\mbx\in \mathcal{X}}    \sum_{t=1}^T f(\mbx, \mby_t) \leq R(T, \delta).$$
\item We call an online learning algorithm $\mathcal{L}_{x}$ for
  $\mbx$ a {\em strong learning algorithm}  with regret $R(\cdot, \cdot)$,  if for any anticipatory sequence $\{\mby_t\}_{t=1}^\infty$,
  the following holds with a probability $1-\delta$ (over the exogenous randomness of the algorithm), where $\{\mbx_t\}$ are the output of $\mathcal{L}_{x}$:
$$ \sum_{t=1}^T f(\mbx_t, \mby_t) -\min_{\mbx\in \mathcal{X}}    \sum_{t=1}^T f(\mbx, \mby_t) \leq R(T, \delta).$$
\end{enumerate}
\end{definition}
To illustrate the subtle difference, let us consider the classical
exponential weighting algorithms, where a set of experts are given,
and the goal of the online learning algorithm is to predict unseen
$\mby_t$ according to the prediction of the experts, such that the
algorithm does as good as the best among the experts. The algorithm
maintains a weight vector over all experts depending on their
performance in previous rounds, and outputs the weighted average of the
prediction from the experts. Notice that this is a deterministic
learning algorithm, i.e., the learning algorithm is independent of the
exogenous randomness $\xi_t$. Thus, whether $\mby_t$ has access to
$\xi_t$ or not has no influence on the performance of the
algorithm. As such, the exponential weighting algorithm (in this
specific form) is a strong learning algorithm.

On the other hand, there is a variant of exponential weighting
algorithm where instead of outputting the weighted average of the
prediction, the algorithm outputs the prediction of one expert, based
on a probability proportional to the weight. The common proof for this
technique is that through {\em randomization}, the {\em expected loss}
is upper bounded by the loss of the weighted average, and hence the
regret of this variant is upper bounded by the regret of the vanilla
version. Clearly, this argument implicitly uses an assumption that the
realized $\mby_t$ is independent of this randomness, and breaks down
otherwise. Hence, this form of exponential weighting algorithm is a
weak learning algorithm.

As a rule of thumb, it appears that for online learning algorithms
that ``work in the adversarial case'', all deterministic algorithms
(e.g., \emph{Online Gradient Descent}) are strong learning algorithms;
whereas all algorithms which inherently require randomness (e.g.,
\emph{Follow the Perturbed Leader}) are weak learning algorithms.

We remark that in the online learning literature, there is the concept
of {\em adaptive adversaries}, which is a relevant concept that can
better highlight the observation made in the paper. An adaptive
adversary in online learning is allowed to adapt its choice at time
$t$ to the output of the online learning algorithm until time $t-1$,
but is independent of the exogenous randomness at time $t$. Thus, it
generates a non-anticipatory sequence. An online learning algorithm that
achieves a diminishing regret against such an adversary is thus a weak
learner, and not necessarily a strong learner.

\section{Warmup: Minimax problem via Online Learning}
\label{sec:robOnl}

We will first consider the case where we have one function \(f\).
In principle this function \(f\) can be highly complex and could be, e.g., the
maximum of a family of functions \(f_i\), however here the reader
should be thinking of \(f\) as a relatively simple function. This will
be made more precise below, where we specify the learnability
requirements for \(f\), which ultimately limits the complexity of the
considered functions. In Section~\ref{sec.multi} we will then consider
the more general case of a family of (simple) functions \(\{f_i\}_i\), which arises naturally in robust optimization. 

Thus, we are solving the following optimization problem
\begin{equation}
  \label{eq:robOptProblem}
\underset{\mbx\in \mathcal{X}}{\mbox{Minimize}}\,\, \underset{\mbu\in \mathcal{U}}{\mbox{Maximize}} \,\, f(\mbx, \mbu).
\end{equation}
\begin{assumption}[Problem structure] \label{ass:setup}We will make the following
  assumptions regarding the domains and function \(f\) if not stated
  otherwise. Note that these assumptions only affect the \(\mbx\)-player. 
(1) For any $\mbu\in \mathcal{U}$, the function $f(\cdot, \mbu)$ is convex.
(2) The set $\mathcal{X}$ is convex.
\end{assumption}

\subsection{Parallel Weak Learners}
Our first framework solves Problem~\eqref{eq:robOptProblem} via weak
online learning algorithms and \emph{imaginary play} (i.e., both
players \emph{can} have full knowledge about the function \(f\)) to update $\mbx$ and
$\mbu$ in parallel. In the following we will always assume that the
\(\mbx\)-sequence is a sequence of elements in \(\mathcal X\) and the
\(\mbu\)-sequence is a sequence of elements in \(\mathcal U\).
 
\begin{assumption}[Weak Learnability] \label{ass:learn}\ 
\begin{enumerate} 
\item There exists a weak online learning algorithm $\mathcal{L}_{x}$
  for $\mbx$ with regret $R_x(\cdot, \cdot)$. That is, for any NONA
  sequence $\{\mbu'_t\}_{t=1}^\infty$, the following holds with a
  probability $1-\delta$, where $\{\mbx_t\}$ is the output of
  $\mathcal{L}_{x}$:
$$ \sum_{t=1}^T f(\mbx_t, \mbu'_t) -\min_{\mbx\in \mathcal{X}}    \sum_{t=1}^T f(\mbx, \mbu'_t) \leq R_x(T, \delta).
$$
\item There exists a weak online learning algorithm
  $\mathcal{L}_{u}$ for $\mbu$ with regret $R_u(\cdot, \cdot)$. That is, for any NONA sequence
  $\{\mbx'_t\}_{t=1}^\infty$, the following holds with a probability
  $1-\delta$, where $\{\mbu_t\}$ is the output of $\mathcal{L}_{u}$:
$$ \max_{\mbu\in \mathcal{U}}    \sum_{t=1}^T f(\mbx'_t, \mbu) - \sum_{t=1}^T f(\mbx'_t, \mbu_t)  \leq R_u(T, \delta).
$$
\end{enumerate}
\end{assumption}

As mentioned above, the learnability assumption constrains the
complexity of the function \(f\). For example if \(f(\mbx,\mbu) = \max_{i \in
  [\ell]} f_i(\mbx,\mbu)\) for some family of functions \(\{f_i\}_i\)
that are convex in \(\mbx\) and concave in \(\mbu\),
then \(f\) might not be concave in \(\mbu\) and the resulting Problem~\eqref{eq:robOptProblem} might be
intractable and the learnability assumption for \(\mbu\) might be
violated. 

We are now ready to present the meta-algorithm, which is given in
Algorithm~\ref{alg:rool}. We would like to remark that we refer to
these minimax problems as \emph{robust optimization} as we only require to be able to
produce an explicit (stationary) solution \(\bar\mbx\) for the
\(\mbx\)-player. 

\begin{algorithm}
  \caption{Robust Optimization via Online Learning (ROOL)}
  \label{alg:rool}
  \begin{algorithmic}[1]
    \REQUIRE function \(f\), learners \(\mathcal L_u,\mathcal L_x\)
    satisfying Assumptions~\ref{ass:setup} and~\ref{ass:learn}. 
    \ENSURE point $\overline{\mbx}$
    \FOR{\(t = 1, \dots, T\)}
      \STATE \(\mbx_t \leftarrow \mathcal L_x (\mbx_1,\mbu_1,\cdots,\mbx_{t-1}, \mbu_{t-1})\)
      \STATE \(\mbu_t \leftarrow \mathcal L_u
      (\mbx_1,\mbu_1,\cdots,\mbx_{t-1}, \mbu_{t-1})\)
    \ENDFOR
      \STATE \(\overline{\mbx} \leftarrow \frac{1}{T}\sum_{t=1}^T \mbx_t\)

  \end{algorithmic}
\end{algorithm}

\begin{remark}[Dependence on \(f\)]
  Note that in Algorithm~\ref{alg:rool} the function \(f\)
  does not explicitly occur. In fact, \(f\) is captured in
  Assumptions~\ref{ass:setup} and~\ref{ass:learn} and in particular,
  we make \emph{a priori} no distinction what type of feedback (full
  information, semi-bandit, bandit, etc.) the learner observes. In
  principle, since we are assuming imaginary play both learners can
  have full knowledge about the function \(f\)
  while in actual applications the learners will only require limited
  information. For example, a learner might only require bandit
  feedback to ensure the learnability assumption with a given regret
  bound, while another might require full information depending on the
  setup. In the formulation above, Algorithm~\ref{alg:rool} is
  completely agnostic to this; also in all other algorithms, the situation
  will be analogous.
\end{remark}
 
Observe that due to convexity of $\mathcal{X}$, we have
  $\overline{\mbx} \in \mathcal{X}$. The theorem below shows that
  $\overline{\mbx}$ converges to $\mbx^*$, which achieves the best
  worst-case performance. Note that the guarantee is asymmetric as a
  saddle point may not exist as no assumptions on $\mathcal{U}$ or
  $f(\mathbf{x},\cdot)$ are made.  If indeed $f$ is concave with
  respect to the second argument and $\mathcal{U}$ is convex, then
  the theorem reduces to the well known result of solving a zero-sum
  game via online learning in parallel~\cite{freund1999adaptive}. The proof is similar
  and included in the supplementary material for completeness.

\begin{theorem} \label{thm:main-result}
With probability $1-2\delta$, Algorithm~\ref{alg:rool} returns a
point $\overline{\mbx}=\frac{1}{T}\sum_{t=1}^T \mbx_t$ satisfying
\begin{align}
  \label{eq:2}
\max_{\mbu'\in \mathcal{U}} f(\overline{\mbx}, \mbu')
-\min_{\mbx^*\in \mathcal{X}} \max_{\mbu\in \mathcal{U}} f(\mbx^*,
\mbu) 
\leq & \frac{ R_x(T, \delta)+R_u(T, \delta)}{T}.  
\end{align}
\end{theorem}

We remark that the {\em two weak learners} framework superficially
resembles the online learning based method to solve zero-sum
games~\cite{freund1999adaptive} where both players run an online
learning algorithm. Yet, our setup and results depart from those
in~\cite{freund1999adaptive}, as we drop any requirement for the
uncertainty \(\mbu\) and in particular \(f\) is not necessarily
concave with respect to \(\mbu\). In short, the minimax problem we
solve is not a saddle-point problem, and as such only the $\mbx$
player is able to extract a near-optimal solution. Notice that the
lack of concavity with respect to \(\mbu\) arises naturally in robust
optimization formulations and adversarial training (see
Section~\ref{sec.multi} for details).

\subsection{Biased Play with a  Strong Learner}
\label{sec:bias}
In our second framework,  the
structure of the problem is ``biased'' toward one player. Here one of
the learners is particularly strong, allowing the other to break NONA-ness. 
We will consider the case where the \(\mbu\) learner is
particularly strong. The case for \(\mbx\) is
symmetric. 

\begin{assumption}[Strong Learnability of $\mbu$] \label{ass:slearn-x}\ 
\begin{enumerate} 
\item There exists a strong online learning algorithm $\mathcal{L}_{u}$ for
  $\mbu$ with regret $R_u(\cdot, \cdot)$. That is,  for any anticipatory sequence $\{\mbx'_t\}_{t=1}^\infty$,
  the following holds with a probability $1-\delta$, where $\{\mbu_t\}$ is the output of $\mathcal{L}_{x}$:
$$ \max_{\mbu}\sum_{t=1}^T  f(\mbx'_t, \mbu) - \sum_{t=1}^T f(\mbx'_t, \mbu_t)   \leq R_u(T, \delta).
$$
\item Given $\mbu$, there is an optimization oracle that computes
$\mbx^*=\arg\min_{\mbx\in \mathcal{X}} f(\mbx, \mbu)$.
\end{enumerate}
\end{assumption}

\begin{algorithm}
  \caption{Robust Optimization via Strong Primal Learner  }
  \label{alg:spl}
  \begin{algorithmic}[1]
    \REQUIRE function \(f\), learners \(\mathcal L_x\)
    satisfying Assumptions~\ref{ass:setup} and~\ref{ass:slearn-x}. 
    \ENSURE point $\overline{\mbx}$
    \FOR{\(t = 1, \dots, T\)}
      \STATE \(\mbu_t \leftarrow \mathcal L_u (\mbx_1,\mbu_1,\cdots,\mbx_{t-1}, \mbu_{t-1})\)
      \STATE \(\mbx_t \leftarrow \arg\min_{\mbx\in \mathcal{X}} f(\mbx, \mbu_t) \)
    \ENDFOR
      \STATE \(\overline{\mbx} \leftarrow \frac{1}{T}\sum_{t=1}^T \mbx_t\)

  \end{algorithmic}
\end{algorithm}

 We now a theorem similar to Theorem~\ref{thm:main-result},
 for the case where \(\mbx\) is a strong learner; the proof is to be
 found in Supplementary Material~\ref{sec:missing-proofs}.
 
\begin{theorem} \label{thm:main-result-strongX} 
With probability $1-\delta$, Algorithm~\ref{alg:spl} returns a
point $\overline{\mbx}=\frac{1}{T}\sum_{t=1}^T \mbx_t$ satisfying
$$\max_{\mbu'\in \mathcal{U}} f(\overline{\mbx}, \mbu') -\min_{\mbx^*\in \mathcal{X}} \max_{\mbu\in \mathcal{U}} f(\mbx^*, \mbu) \leq \frac{ R_u(T, \delta)}{T}.
$$
\end{theorem}

Some remarks are in order.
\begin{enumerate}
\item Note that since $\mbu_t$ is updated via solving an optimization
  problem determined by $\mbx_t$, the sequence $\{\mbu_t\}$ is an
  anticipatory sequence. As such, it is crucial that a strong learning
  algorithm is required to update $\mbx_t$. This is explains the
  existence of the counter example in Section~\ref{sub.counter}: both
  FPL and exponential weighting (with output randomly chosen) are weak
  learners, as opposed to online gradient descent which is a strong learner.
  \item It is easy to extend the analysis to the case where an optimization oracle is replaced by an $C$-approximate optimization oracle, which given $\mbu$ computes $\tilde{\mbx}^*$ such that
  $$ f(\tilde{\mbx}^*, \mbu) \leq C\cdot \min_{\mbx\in \mathcal{X}} f(\mbx, \mbu),
  $$
  for some $C>1$. In such a case, the statement in Theorem~\ref{thm:main-result-strongX}  is replaced by
  $$\max_{\mbu'\in \mathcal{U}} f(\overline{\mbx}, \mbu') -C\cdot\min_{\mbx^*\in \mathcal{X}} \max_{\mbu\in \mathcal{U}} f(\mbx^*, \mbu) \leq \frac{ R_u(T, \delta)}{T}.
$$
That is, the algorithm will return a $C$-approximate solution for the
minimax problem.
\end{enumerate} 

Before concluding this section, we remark that the convexity
requirement for the $\mbx$ player can be relaxed if randomized actions
are allowed; see Appendix~\ref{sec:randomized-case} for
details. Consequently, this allows us to solve robust optimization
where both the feasible region $\mathcal{X}$ and the uncertainty set
$\mathcal{U}$ are represented as feasible regions of Integer
Programming problems (see Appendix~\ref{sec:robust-optim-with} for an
example).
 
\section{Multiple Objectives: Online Learning for Robust Optimization}\label{sec.multi}
Our general approach can be readily extended to the case where the
primal player needs to satisfy multiple objectives simultaneously.
Multi-objective decision-making naturally arises in many setups where
the preference of decisions are multi-dimensional. In particular
multi-objective decision-making can model robust optimization, where
typically the decision maker aims to find a decision $\mbx^*$ such
that a set of robust constraints are satisfied, i.e.,
$$\underset{\mbu^i\in \mathcal{U}^i}{\mbox{Maximize}}f^i(\mbx^*,\mbu^i) \leq 0,\quad i=1,2,\cdots, n.
$$

 We consider solving the following general case:
 \begin{equation}
  \label{eq:robOptProblem-multi}
\underset{\mbx\in \mathcal{X}}{\mbox{Minimize}} \,\,
\left\{\underset{\boldsymbol{\lambda} \in \Lambda }{\mbox{Maximize}}  \underset{\mbu^i\in \mathcal{U}^i, i=1,\cdots, n}{\mbox{Maximize}} \,\, \sum_{i=1}^n \lambda^i f^i(\mbx, \mbu^i)\right\},
\end{equation}
where $\Lambda\subseteq \Delta(n)$ is a closed convex set. For
example, if $\Lambda=\Delta(n)$, the \(n\)-dimensional unit simplex, then Problem~\eqref{eq:robOptProblem-multi} reduces to
\begin{equation*}\begin{split}
&\underset{\mbx\in \mathcal{X}}{\mbox{Minimize}}\,\,   \max \left(\underset{\mbu^1\in \mathcal{U}^1}{\mbox{Maximize}}    f^1(\mbx, \mbu^1),\right.\\ 
&\qquad\qquad\left.  \underset{\mbu^2\in \mathcal{U}^2 }{\mbox{Maximize} } f^2(\mbx, \mbu^2),\cdots, \underset{\mbu^n\in \mathcal{U}^n}{\mbox{Maximize}}   f^n(\mbx, \mbu_n)\right),
\end{split}\end{equation*}
which corresponds to the aforementioned case of robust
optimization. On the other hand, if $\Lambda=\{\boldsymbol{\lambda}\}$
is a singleton, then \eqref{eq:robOptProblem-multi} is equivalent to
$$ \underset{\mbx\in \mathcal{X}}{\mbox{Minimize}} \,\, \underset{\mbu^i\in \mathcal{U}^i, i=1,\cdots, n}{\mbox{Maximize}} \left(\sum_{i=1}^n \lambda^i f^i(\mbx, \mbu^i)\right),
$$
and we solve our problem for a specific preference or weighing among
the different objectives.

We first consider solving Problem~\eqref{eq:robOptProblem-multi} via
parallel weak learners. We present the following two approaches both
of which are based on imaginary play. Due to space constraints, the
biased case with a strong learner is deferred to the supplementary
material.

\subsubsection*{Approach via Explicit Maximum}
In the first approach we model the maximum over the different functions
\(f^i\)
explicitly. To this end, let $\vec{\mbu}\in \prod_i \mathcal{U}^i$
denote the concatenation of $\mbu^i$, i.e.,
$\vec{\mbu} =\{\mbu^1, \mbu^2,\cdots \mbu^n\} $ and define the function
$$F(\mbx, \vec{\mbu}) \triangleq \max_{\boldsymbol{\lambda}\in \Lambda} \lambda^i f^i(\mbx, \mbu^i).$$
Thus, roughly speaking, the optimal $\mbx$ is approachable if weak learnability holds for both $\mbx$ and $\mbu$ with respect to $F$.
Due to space constraints we defer detailed results into the supplementary material.

\subsubsection*{Approach via Distributional Maximum}\label{sss.distributional_maximum}
In this section we will present an alternative approach, where the
maximum is only implicitly modeled via a distributional approach,
which captures the maximum via a worst-case distribution. 

In the following let 
$$g(\mbx, \vec{\mbu}, \blambda) \triangleq \langle \blambda, f(\mbx, \vec{\mbu})\rangle.$$
With this Problem~\eqref{eq:robOptProblem-multi} can be rewritten as
$$\underset{\mbx\in \mathcal{X}}{\mbox{Minimize}}\,\, \underset{\vec{\mbu}\in \prod_i \mathcal{U}^i, \blambda\in \Lambda}{\mbox{Maximize}}\,\, g(\mbx, \vec{\mbu}, \blambda). 
$$
As before we specify the learnability requirement.

\begin{assumption}[Learnability]\label{assu.lambdalearn} We make the
  following assumptions for the learners:
\begin{enumerate} 
\item For every \(i\), there exists an online learning algorithm $\mathcal{L}^i_{u}$
  for $\mbu_i$ for $f^i(\cdot)$, i.e., for any NONA sequence of
  $\{\mbx'_t\}$ , the following holds with probability $1-\delta$:
$$ \max_{\mbu^i\in \mathcal{U}^i}    \sum_{t=1}^T f^i(\mbx'_t, \mbu^i) - \sum_{t=1}^T f(\mbx'_t, \mbu_t^i)  \leq R_u^i(T, \delta),
$$
where $\{\mbu^i_t\}$ is the output of $\mathcal{L}^i_{u}$.
\item There exists an online learning algorithm
  $\mathcal{L}_{\lambda}$ for $\blambda$ for $g(\cdot)$. That is, for
  any NONA sequences of $\{\mbx'_t\}$ and $\{\vec{\mbu}'_t\}$, the
  following holds with probability $1-\delta$:
$$\max_{\blambda\in \Lambda} \sum_{t=1}^T \langle \blambda, f(\mbx'_t,\vec{\mbu}'_t) \rangle -  \sum_{t=1}^T \langle \blambda_t, f(\mbx'_t,\vec{\mbu}'_t) \rangle \leq R_{\lambda}(T,\delta),
$$
where  $\{\blambda_t\}$ is the output of $\mathcal{L}_{\lambda}$.
\item There exists an online learning algorithm
  $\mathcal{L}_{x}$ for $\mbx$ for $g(\cdot)$. That is, for
  any NONA sequences of $\{\mbu'_t\}$ and $\{\blambda'_t\}$, the
  following holds with probability $1-\delta$:
$$\sum_{t=1}^T \langle \blambda'_t, f(\mbx_t,\vec{\mbu}'_t) \rangle - \min_{\mbx \in \mathcal X} \sum_{t=1}^T \langle \blambda'_t, f(\mbx,\vec{\mbu}'_t) \rangle  \leq R_{x}(T,\delta),
$$
where  $\{\mbx_t\}$ is the output of $\mathcal{L}_{x}$.
\end{enumerate}
\end{assumption}

Note that the \(\blambda\)-learner and \(\mbu\)-learner should be
considered as the dual learners and \(\mbx\) as the primal learner. In
fact, we show that the \(\blambda\)-learner and \(\mbu\)-learner
together give rise to a \((\vec{\mbu},\blambda)\)-learner, which allows us
then to reuse previous methodology. 
Further observe that,
$\langle \blambda, f(\mbx, \vec{\mbu})\rangle$ is a linear function of
$\blambda$ and thus the second part of the assumption, for example,
holds using the Follow the Perturbed Leader algorithm (see
\cite{kalai2005efficient}).

\begin{proposition}\label{pro.multi.dual}
  Suppose that Assumption~\ref{assu.lambdalearn} holds and that
  \(\Lambda \subseteq \Delta(n)\). Then running
  $\mathcal{L}^i_u$ and $\mathcal{L}_{\lambda}$ simultaneously is an
  online learning algorithm for $(\vec{\mbu}, \blambda)$ of function
  $g(\cdot,\cdot)$. That is, for any NONA sequence $\{\mbx'_t\}$, let
  $\{\mbu^i_t\}$ be the output of $\mathcal{L}^i_{u}$, and
  $\{\blambda_t\}$ be the output of $\mathcal{L}_{\lambda}$, then with probability $1-(n+1)\delta$, we have
\begin{eqnarray*}
&&\max_{\blambda\in \Lambda, \vec{\mbu}\in \prod_i\mathcal{U}^i} \sum_{t=1}^T \langle \blambda, f(\mbx'_t, \vec{\mbu}) \rangle -  \sum_{t=1}^T \langle \blambda_t, f(\mbx'_t, \vec{\mbu}_t)\\
&& \qquad \leq \max_i  R^i_{\mbu}(T,\delta)  +R_{\blambda}(T, \delta).
\end{eqnarray*}
\end{proposition}

By Proposition~\ref{pro.multi.dual}, there exist weak learners for
both the primal and the dual player and solving
Problem~\ref{eq:robOptProblem-multi} reduces to solving
Problem~\ref{eq:robOptProblem}. Below we present the formal algorithm
and the corresponding theorem with performance guarantees.

\begin{algorithm}
  \caption{Robust Optimization via Online Learning (ROOL) for
    maximum over functions (adaptive $\lambda$)}
  \label{alg:rool-const-learnlambda}
  \begin{algorithmic}[1]
    \REQUIRE function \(f^1,\cdots, f^n\), learners \(\mathcal L_u^1,\cdots,  \mathcal L_u^n, \mathcal L_{\lambda}, \mathcal L_x\)
    satisfying Assumptions~\ref{ass:setup} and~\ref{assu.lambdalearn}. 
    \ENSURE point $\overline{\mbx}$
    \FOR{\(t = 1, \dots, T\)}
      \STATE \(\mbx_t \leftarrow \mathcal L_x (\mbx_1,\vec{\mbu}_1, \boldsymbol{\lambda}_1, \cdots,\mbx_{t-1}, \vec{\mbu}_{t-1}, \boldsymbol{\lambda}_{t-1})\)
      \STATE \(\mbu^i_t \leftarrow \mathcal L^i_u
      (\mbx_1,\mbu^i_1, \cdots,\mbx_{t-1}, \mbu^i_{t-1});\quad i=1,\cdots, n\)
      \STATE \(\boldsymbol{\lambda}_t \leftarrow \mathcal L_{\lambda} (\mbx_1, \vec{\mbu}_1,\boldsymbol{\lambda}_1,\cdots, \mbx_{t-1}, \vec{\mbu}_{t-1}, \boldsymbol{\lambda}_{t-1}) \)
    \ENDFOR
      \STATE \(\overline{\mbx} \leftarrow \frac{1}{T}\sum_{t=1}^T \mbx_t\)
  \end{algorithmic}
\end{algorithm}

\begin{theorem}\label{thm.multi.dist.main}  Suppose that Assumption~\ref{ass:setup} and~\ref{assu.lambdalearn} hold and that
  \(\Lambda \subseteq \Delta(n)\). Then with probability
  $1-(n+2)\delta$, Algorithm~\ref{alg:rool-const-learnlambda} returns
  a point $\bar{\mbx}$ satisfying
\begin{equation*}\begin{split}
&  \max_{\vec{\mbu}'\in \prod_i \mathcal{U}^i, \blambda'\in \Lambda} g(\overline{\mbx}, \vec{\mbu}', \blambda')- \min_{\mbx\in \mathcal{X}}   \max_{\vec{\mbu}\in \prod_i \mathcal{U}^i, \blambda\in \Lambda} g(\mbx, \vec{\mbu}, \blambda)\\
&\quad \le  \frac{R_x(T,\delta) + \max_i R^i_u(T, \delta) +R_{\lambda}(T, \delta)}{T}.
\end{split}\end{equation*}
\end{theorem}

To illustrate the result, let us consider the following
example. Suppose all $f_i(\mbx, \mbu^i)$ are bilinear with respect to
$\mbx$ and $\mbu^i$, as in the case of a robust linear programming,
$\mathcal{X}$ and $\mathcal{U}^i$ are subsets of the Euclidean space, and
further suppose $\mathcal{X} $ is a convex set (notice that we make no
such assumptions on the uncertainty sets $\mathcal{U}_i$). We say a set
$\mathcal{Z}\subset \mathbb{R}^m$ is equipped with a linear
optimization oracle, if given any
$\boldsymbol{\theta}\in \mathbb{R}^m$, we can compute
$\underset{\mathbf{z}\in \mathcal{Z}}{\mbox{Minimize}} \quad
\boldsymbol{\theta}^\top \mathbf{z}.  $ Thus,
Assumption~\ref{assu.lambdalearn} holds if the followings are true:
 \begin{enumerate}
 \item $\mathcal{U}^i$ is equipped with a linear optimization oracle for $i=1,\cdots, n$.
 \item $\Lambda$ is equipped with a linear optimization oracle.
 \item $\mathcal{X}$ is equipped with a linear optimization oracle. 
 \end{enumerate}

 Indeed, each of the three conditions ensures the learnability in
 Assumption~\ref{assu.lambdalearn} for $\mbu^i$, $\blambda$, and
 $\mbx$ respectively, via e.g., the Follow the Perturbed Leader
 algorithm. This is due to the fact that for each argument, its
 respective objective function is linear.

\bibliography{bibliography}
\bibliographystyle{abbrvnat}

\newpage

\appendix

\section{Proofs}
\label{sec:missing-proofs}

\subsection{Proofs from Section~\ref{sec:robOnl}}

  \begin{proof}[Proof of Theorem~\ref{thm:main-result}] Since $\mbx_t$ and $\mbu_t$ are obtained by
  $\mathcal{L}_{x}$ and $\mathcal{L}_{u}$, we have that $\{\mbx_t\}_{t}$
  and $\{\mbu_t\}_{t}$ are NONA. Thus, by Assumption~\ref{ass:learn},
\begin{eqnarray}
&& \sum_{t=1}^T f(\mbx_t, \mbu_t) -\min_{\mbx\in \mathcal{X}}    \sum_{t=1}^T f(\mbx, \mbu_t) \leq R_x(T, \delta),
\end{eqnarray}
and 
\begin{eqnarray}
&&   \max_{\mbu\in \mathcal{U}}    \sum_{t=1}^T f(\mbx_t, \mbu) - \sum_{t=1}^T f(\mbx_t, \mbu_t)  \leq R_u(T, \delta),
\end{eqnarray}
hold simultaneously with probability $1-2\delta$. Summing up the two inequalities leads to 
\begin{equation}
\label{equ.onlinegap}
\begin{split}
&\max_{\mbu\in \mathcal{U}}    \sum_{t=1}^T f(\mbx_t, \mbu)-\min_{\mbx\in \mathcal{X}}    \sum_{t=1}^T f(\mbx, \mbu_t)\\ &  \leq  R_x(T, \delta)+R_u(T, \delta).
\end{split}
\end{equation}
By convexity of $f(\cdot, \mbu)$, we have
$T f(\overline{\mbx}, \mbu) \leq  \sum_{t=1}^T f(\mbx_t, \mbu)$ for
all $\mbu\in \mathcal{U}$,
so that
\begin{equation}
\label{equ:xgap}
T\max_{\mbu'\in \mathcal{U}} f(\overline{\mbx}, \mbu') \leq \max_{\mbu\in \mathcal{U}}  \sum_{t=1}^T f(\mbx_t, \mbu).
\end{equation}
Since $\mbu_t\in \mathcal{U}$ for all \(t\)
we also have for any $\mathbf{x}$ that
$f(\mbx, \mbu_t)\leq \max_{\mbu\in \mathcal{U}} f(\mbx, \mbu)$, which
implies
$\sum_{t=1}^T f(\mbx, \mbu_t)\leq  T\max_{\mbu\in \mathcal{U}} f(\mbx, \mbu),
$
further leading to
\begin{equation}
\label{equ:ugap} 
\min_{\mbx\in \mathcal{X}}   \sum_{t=1}^T f(\mbx, \mbu_t)\leq T  \min_{\mbx^*\in \mathcal{X}}    \max_{\mbu\in \mathcal{U}} f(\mbx^*, \mbu).
\end{equation}
Combining Equations~\eqref{equ:xgap} and \eqref{equ:ugap} we obtain
\begin{equation*}
\begin{split}
&T\max_{\mbu'\in \mathcal{U}} f(\overline{\mbx}, \mbu')  -  T  \min_{\mbx^*\in \mathcal{X}}    \max_{\mbu\in \mathcal{U}} f(\mbx^*, \mbu)\\
 \leq &  \max_{\mbu\in \mathcal{U}}  \sum_{t=1}^T f(\mbx_t, \mbu) - \min_{\mbx\in \mathcal{X}}   \sum_{t=1}^T f(\mbx, \mbu_t),
\end{split}
\end{equation*}
which together with~\eqref{equ.onlinegap} establishes the theorem.
\end{proof}

\begin{proof}[Proof of Theorem~\ref{thm:main-result-strongX}]  By Assumption~\ref{ass:slearn-x}, with probability $1-\delta$ we have
\begin{eqnarray*}
&& \max_{\mbu\in \mathcal{U}} \sum_{t=1}^Tf(\mbx_t, \mbu)- \sum_{t=1}^T f(\mbx_t, \mbu_t)   \leq R_u(T, \delta).
\end{eqnarray*}
And by definition of $\mbx_t$, 
\begin{equation*}\begin{split} 
\sum_{t=1}^T f(\mbx_t, \mbu_t)=\sum_{t=1}^T \min_{\mbx\in \mathcal{X}}f(\mbx, \mbu_t)\leq \min_{\mbx \in \mathcal{X}} \sum_{t=1}^T f(\mbx, \mbu_t).
\end{split}\end{equation*}
 Combining the two inequalities leads to 
\begin{equation*}
\max_{\mbu\in \mathcal{U}}    \sum_{t=1}^T f(\mbx_t, \mbu)-\min_{\mbx\in \mathcal{X}}    \sum_{t=1}^T f(\mbx, \mbu_t) \leq  R_u(T, \delta).
\end{equation*}
The argument follows now identically to the proof of Theorem~\ref{thm:main-result}.
\end{proof}

We obtain an analogous statement for \emph{Randomized Robust
  Optimization via Strong Primal Learner} whose proof is almost identical to
the proof of Theorem~\ref{thm:main-result-strongX} from above.

\begin{algorithm}
  \caption{Randomized Robust Optimization via Strong Primal Learner}
  \label{alg:rspl}
  \begin{algorithmic}[1]
    \REQUIRE function \(f\), learners \(\mathcal L_x\)
    satisfying ~\ref{ass:slearn-x}. 
    \ENSURE distribution $\overline{\mu}$
    \FOR{\(t = 1, \dots, T\)}
      \STATE \(\mbx_t \leftarrow \mathcal L_x (\mbx_1,\mbu_1,\cdots,\mbx_{t-1}, \mbu_{t-1})\)
      \STATE \(\mbu_t \leftarrow \arg\max_{\mbu\in \mathcal{U}} f(\mbx_t, \mbu) \)
    \ENDFOR
 \STATE \(\overline{\mu} \leftarrow \mbox{empirical distribution of } (\mbx_1, \cdots, \mbx_T)\)
 
  \end{algorithmic}
\end{algorithm}
 \begin{theorem} \label{thm:main-result-distribution-strongx}
With probability $1-\delta$, Algorithm~\ref{alg:rspl} returns a
distribution $\overline{\mu}$ satisfying
\begin{equation*}\begin{split}
&\max_{\nu'\in \Delta(\mathcal{U})} \underset{{\mbx\sim \overline{\mu}, \mbu\sim \nu'} }{\mathbb{E}}f(\mbx, \mbu)\\&\quad -\min_{\mu^* \in \Delta(\mathcal{X}) } \max_{\nu\in \Delta(\mathcal{U})} \underset{{\mbx\sim \mu^*, \mbu\sim \nu}}{\mathbb{E}} f(\mbx, \mbu)
 \leq \frac{ R_x(T, \delta)}{T}.
\end{split}\end{equation*}
\end{theorem}

\subsection{Proofs from Section~\ref{sec.multi}}
\begin{proof}[Proof of Proposition~\ref{pro.multi.dual}]By Assumption~\ref{assu.lambdalearn}, with probability
  $1-(n+1)\delta$, we have for all \(i\)
\begin{align}\label{equ.proof.constrait.approach2}
  &\max_{\mbu^i\in \mathcal{U}^i}    \sum_{t=1}^T f^i(\mbx'_t, \mbu^i) - \sum_{t=1}^T f(\mbx'_t, \mbu_t^i)  \leq R_u^i(T, \delta),\\
  \text{and} \quad &\max_{\blambda\in \Lambda} \sum_{t=1}^T \langle \blambda, f(\mbx'_t, \vec{\mbu_t}) \rangle -  \sum_{t=1}^T \langle \blambda_t, f(\mbx'_t, \vec{\mbu}_t) \leq R_{\lambda}(T, \delta).
\end{align}
With the above the following holds:
\begin{equation*}
\begin{split}
&\max_{\blambda\in \Lambda, \vec{\mbu}\in \prod_i\mathcal{U}^i} \sum_{t=1}^T \langle \blambda, f(\mbx'_t, \vec{\mbu}) \rangle -  \sum_{t=1}^T \langle \blambda_t, f(\mbx'_t, \vec{\mbu}_t) \rangle \\
\leq & \max_{\blambda \in \Lambda}  \left[ \max_{ \vec{\mbu}\in \prod_i\mathcal{U}^i}\sum_{t=1}^T\langle \blambda, f(\mbx'_t, \vec{\mbu}) \rangle - \sum_{t=1}^T \langle \blambda, f(\mbx'_t, \vec{\mbu}_t) \rangle\right] +  \left[ \max_{\blambda \in \Lambda}\sum_{t=1}^T\langle \blambda, f(\mbx'_t, \vec{\mbu}_t) \rangle - \sum_{t=1}^T \langle \blambda_t, f(\mbx'_t, \vec{\mbu}_t) \rangle\right] \\
 {= }& \max_{\blambda  \in \Lambda}  \left[ \max_{ \vec{\mbu}\in \prod_i\mathcal{U}^i}\left( \sum_{i=1}^n \blambda_i \left(\sum_{t=1}^T   f^i(\mbx'_t, \mbu^i)   - \sum_{t=1}^T   f^i(\mbx'_t, \mbu^i_t) \right) \right)\right] +  \left[ \max_{\blambda  \in \Lambda}\sum_{t=1}^T\langle \blambda, f(\mbx'_t, \vec{\mbu}_t) \rangle - \sum_{t=1}^T \langle \blambda_t, f(\mbx'_t, \vec{\mbu}_t) \rangle\right]\\
= & \max_{\blambda  \in \Lambda}  \left[  \left( \sum_{i=1}^n \blambda_i \left(\max_{\mbu^i\in \mathcal{U}^i}\sum_{t=1}^T   f^i(\mbx'_t, \mbu^i)   - \sum_{t=1}^T   f^i(\mbx'_t, \mbu^i_t) \right) \right)\right] +  \left[ \max_{\blambda  \in \Lambda}\sum_{t=1}^T\langle \blambda, f(\mbx'_t, \vec{\mbu}_t) \rangle - \sum_{t=1}^T \langle \blambda_t, f(\mbx'_t, \vec{\mbu}_t) \rangle\right]\\
\stackrel{(a)}{\leq} & \max_{\blambda\in \Lambda}  \left(\sum_{i=1}^n \blambda_i  R^i_{u}(T,\delta)\right) +R_{\lambda}(T, \delta)\\
\stackrel{(b)}{\leq} & \max_i  R^i_{u}(T,\delta)  +R_{\lambda}(T, \delta),
\end{split}
\end{equation*}
where (a) follows with~\eqref{equ.proof.constrait.approach2} and (b)
holds from $\Lambda\subseteq \Delta(n)$, which completes the proof.
\end{proof}
\begin{proof}[Proof of Theorem~\ref{thm.multi.dist.main}]
Observe that $\{\vec{\mbu}_t\}$, $\{\mbx_t\}$ and $\{\blambda_t\}$ are all NONA.
By Proposition~\ref{pro.multi.dual} we have 
 $$\max_{\blambda\in \Lambda, \vec{\mbu}\in \prod_i\mathcal{U}^i} \sum_{t=1}^T g(\mbx_t, \vec{\mbu},  \blambda)
  -  \sum_{t=1}^T g(\mbx_t, \vec{\mbu}_t,  \blambda_t)
   \leq \max_i  R^i_{\mbu}(T,\delta)  +R_{\blambda}(T, \delta),
 $$
 and by Assumption~\ref{assu.lambdalearn} (iii) we have
$$\sum_{t=1}^T  g(\mbx_t, \vec{\mbu}_t,  \blambda_t) - \min_{\mbx \in \mathcal X} \sum_{t=1}^T g(\mbx, \vec{\mbu}_t,  \blambda_t) \leq R_{x}(T,\delta),
$$
The rest of the proof follows the proof of Theorem~\ref{thm:main-result} and is omitted to avoid redundancy.
\end{proof}

\section{Randomized Case}
\label{sec:randomized-case}
Our main framework can be extended to the case where the $\mbx$ player
is allowed to randomize her action. Indeed, under such a setup, we can
further lift the convexity requirement for $f(\cdot, \mbu)$ and
$\mathcal{X}$. Specifically, we consider solving the following
optimization problem

\begin{equation}
\begin{split}
  \label{eq:robOptProblem-expected}
&\underset{\mu \in \Delta(\mathcal{X})}{\mbox{Minimize}}\,\,\underset{\nu \in \Delta(\mathcal{U})}{\mbox{Maximize}}  \,\ \mathbb{E}_{\mbx\sim \mu, \mbu\sim \nu} f(\mbx, \mbu)\\
=  &\underset{\mu \in \Delta(\mathcal{X})}{\mbox{Minimize}}\,\, \underset{\mbu \in \mathcal{U}}{\mbox{Maximize}}  \,\ \mathbb{E}_{\mbx\sim \mu} f(\mbx, \mbu)
\end{split}
\end{equation}
This can be useful in cases where the decision maker wants to make her
decision against an unknown, but {\em non-adaptive} environment; or
even against an adversarial agent, as long as the adversarial agent
does not take his action after observing the decision maker's
action, i.e., we allow for an oblivious (true) adversary not to be
confused with the learnability assumption and the implied imaginary adversary. 

In the setup
described above, if the decision maker chooses a randomized
policy following a distribution $\mu$, then the expected loss she
incurs is upper bounded by
$\max_{\mbu \in \mathcal{U}} \,\ \mathbb{E}_{\mbx\sim \mu} f(\mbx,
\mbu)$,
which is at least as small as, and can be significantly smaller than
the more pessimistic upper bound
$ \mathbb{E}_{\mbx\sim \mu} \max_{\mbu \in \mathcal{U}}f(\mbx,
\mbu)$.
We next show that our method can readily be applied to this setup
while relaxing the convexity requirement for $\mathcal{X}$ and
$f(\cdot, \mbu)$.

\begin{algorithm}
  \caption{Randomized Robust Optimization via Online Learning (R2OOL)}
  \label{alg:r2ool}
  \begin{algorithmic}[1]
    \REQUIRE function \(f\), learners \(\mathcal L_u,\mathcal L_x\)
    satisfying~Assumption \ref{ass:learn}. 
    \ENSURE distribution $\overline{\mu}$ \\
\emph{Run Algorithm~\ref{alg:rool} with changed output}
      \STATE \(\overline{\mu} \leftarrow \mbox{empirical distribution of } (\mbx_1, \cdots, \mbx_T)\)

  \end{algorithmic}
\end{algorithm}

\begin{theorem} \label{thm:main-result-distribution} 
With probability $1-2\delta$, Algorithm~\ref{alg:r2ool} returns a
distribution $\overline{\mu}$ satisfying
\begin{equation*}
\begin{split}
  &\max_{\nu'\in \Delta(\mathcal{U})} \underset{\mbx\sim \overline{\mu},\mbu\sim \nu'}{\mathbb{E} } f(\mbx, \mbu) -\min_{\mu^* \in \Delta(\mathcal{X}) } \max_{\nu\in \Delta(\mathcal{U})} \underset{\mbx\sim \mu^*, \mbu\sim \nu}{ \mathbb{E}}  f(\mbx, \mbu)\\
  & \leq \frac{ R_x(T, \delta)+R_u(T, \delta)}{T}.
\end{split}
\end{equation*}
\end{theorem}
\begin{proof}[Proof of Theorem~\ref{thm:main-result-distribution}] Since $\mbx_t$ and $\mbu_t$ are obtained by
  $\mathcal{L}_{x}$ and $\mathcal{L}_{u}$, we have that $\{\mbx_t\}_{t}$
  and $\{\mbu_t\}_{t}$ are NONA. Thus, by Assumption~\ref{ass:learn},
\begin{eqnarray}
&& \sum_{t=1}^T f(\mbx_t, \mbu_t) -\min_{\mbx\in \mathcal{X}}    \sum_{t=1}^T f(\mbx, \mbu_t) \leq R_x(T, \delta)\\
&&   \max_{\mbu\in \mathcal{U}}    \sum_{t=1}^T f(\mbx_t, \mbu) - \sum_{t=1}^T f(\mbx_t, \mbu_t)  \leq R_u(T, \delta),
\end{eqnarray}
holds simultaneously with probability $1-2\delta$. Summing up the two inequalities leads to 
\begin{equation}
\label{equ.onlinegap-distr}\begin{split}
&\max_{\mbu\in \mathcal{U}}    \sum_{t=1}^T f(\mbx_t, \mbu)-\min_{\mbx\in \mathcal{X}}    \sum_{t=1}^T f(\mbx, \mbu_t)\\
& \leq  R_x(T, \delta)+R_u(T, \delta).
\end{split}\end{equation}
Since $\overline{\mu}$ is the empirical distribution of $(\mbx_1,\cdots, \mbx_T)$, we have
\begin{equation*}\begin{split} 
&\max_{\mbu\in \mathcal{U}}  \sum_{t=1}^T f(\mbx_t, \mbu) = T \max_{\mbu\in \mathcal{U}}  \mathbb{E}_{\mbx\sim \overline{\mu} }  f(\mbx, \mbu)\\
& = T \max_{\nu'\in \Delta(\mathcal{U})}  \mathbb{E}_{\mbx\sim \overline{\mu}, \mbu\sim \nu' }  f(\mbx, \mbu).
\end{split}\end{equation*}
Further let $\overline{\nu}$ be the empirical distribution of $(\mbu_1,\cdots, \mbu_T)$. Then we have
\begin{equation*}\begin{split}
&\min_{\mbx\in \mathcal{X}}    \sum_{t=1}^T f(\mbx, \mbu_t) = T \min_{\mbx\in \mathcal{X}} \mathbb{E}_{\mbu \sim \overline{\nu}}f(\mbx, \mbu)\\ &= T \min_{\mu^*\in \Delta(\mathcal{X})} \mathbb{E}_{\mbx\sim \mu^*,\mbu \sim \overline{\nu}}f(\mbx, \mbu)\\
& \leq T\min_{\mu^*\in \Delta(\mathcal{X})} \max_{\nu\in \Delta(\mathcal{U})} \mathbb{E}_{\mbx\sim \mu^*,\mbu \sim \nu}f(\mbx, \mbu). 
\end{split}\end{equation*}
Substituting the two equations to Equation~\eqref{equ.onlinegap-distr}
establishes the theorem.
\end{proof}

As before note that the second summand of the statement in the theorem
is Problem~\eqref{eq:robOptProblem-expected}.  We remark that
Problem~\eqref{eq:robOptProblem-expected} cannot be reduced to
Problem~\eqref{eq:robOptProblem} with the decisions being $\mu$ and $\nu$, and
the objective function being
$g(\mu, \nu)=\mathbb{E}_{\mbx\sim \mu} f(\mbx, \mbu)$; in particular
Algorithm~\ref{alg:rool} does not apply. This is due to
two key differences: first we do not assume that $\mu$ and $\nu$ are
learnable. Furthermore, for given $\mu$ and $\nu$ we do not observe
the value of $g(\mu, \nu)$, instead, we only observe a noisy
realization whose expected value equals $g(\mu, \nu)$.

\section{Multiple Objectives: Online Learning for Robust
  Optimization \\  --- Additional Results ---}\label{sec.multiadd}
In this appendix we provide omitted details for Section~\ref{sec.multi}.

\subsection{Approach via Explicit Maximum}
In the first case we model the maximum over the different functions
\(f^i\)
explicitly. To this end, let $\vec{\mbu}\in \prod_i \mathcal{U}^i$
denote the concatenation of $\mbu^i$, i.e.,
$\vec{\mbu} =\{\mbu^1, \mbu^2,\cdots \mbu^n\} $ and define the function
$$F(\mbx, \vec{\mbu}) \triangleq \max_{\boldsymbol{\lambda}\in \Lambda} \lambda^i f^i(\mbx, \mbu^i).$$
As before we have to specify the learnability for the two players.
\begin{assumption}[Learnability] \label{ass:learn_const1}\ 
\begin{enumerate} 
\item There exists an online learning algorithm $\mathcal{L}_{x}$ for
  $\mbx$, such that for any NONA sequence $\{\vec{\mbu}'_t\}_{t=1}^\infty$,
  the following holds with a probability $1-\delta$
$$ \sum_{t=1}^T F(\mbx_t, \vec{\mbu}'_t) -\min_{\mbx\in \mathcal{X}}    \sum_{t=1}^T F(\mbx, \vec{\mbu}'_t) \leq R_x(T, \delta),
$$
where $\{\mbx_t\}$ is the output of $\mathcal{L}_{x}$.
\item For each  $i=1,\cdots, n$, there exists an online learning algorithm
  $\mathcal{L}^i_{u}$ for $\mbu^i$, such that for any NONA sequence
  $\{\mbx'_t\}_{t=1}^\infty$, the following holds with a probability
  $1-\delta$
$$ \max_{\mbu^i\in \mathcal{U}^i}    \sum_{t=1}^T f^i(\mbx'_t, \mbu^i) - \sum_{t=1}^T f(\mbx'_t, \mbu_t^i)  \leq R_u^i(T, \delta),
$$
where $\{\mbu^i_t\}$ is the output of $\mathcal{L}^i_{u}$.
\end{enumerate}
\end{assumption}

 It is important to observe that for the $\mbu$ player, we do not require learnability for \(F\) the
maximum 
over the \(f^i\)
but only the learnability for each separate \(f^i\).
In fact, as we will see later, these two notions of learnability are
not equivalent. In the remainder of this subsection we will work under
Assumptions~\ref{ass:setup} and~\ref{ass:learn_const1}.

We are ready to formulate our algorithm, which is similar in spirit to
Algorithm~\ref{alg:rool}.

\begin{algorithm}
  \caption{Robust Optimization via Online Learning (ROOL) for
    maximum over functions (optimal $\lambda$)}
  \label{alg:rool-const}
  \begin{algorithmic}[1]
    \REQUIRE function \(f^1,\cdots, f^n\), learners \(\mathcal L_u,\mathcal L_x\)
    satisfying Assumptions~\ref{ass:setup} and~\ref{ass:learn_const1}. 
    \ENSURE point $\overline{\mbx}$
    \FOR{\(t = 1, \dots, T\)}
      \STATE \(\mbx_t \leftarrow \mathcal L_x (\mbx_1,\vec{\mbu}_1,\cdots,\mbx_{t-1}, \vec{\mbu}_{t-1})\)
      \STATE \(\mbu^i_t \leftarrow \mathcal L^i_u
      (\mbx_1,\mbu^i_1,\cdots,\mbx_{t-1}, \mbu^i_{t-1});\quad i=1,\cdots, n\)
    \ENDFOR
      \STATE \(\overline{\mbx} \leftarrow \frac{1}{T}\sum_{t=1}^T \mbx_t\)
  \end{algorithmic}
\end{algorithm}

We establish the following guarantee.

\begin{theorem}\label{thm.constraint-app1-main} With probability
  $1-(n+1)\delta$, Algorithm~\ref{alg:rool-const} returns a point
  \(\overline{\mbx} = \frac{1}{T}\sum_{t=1}^T \mbx_t\) satisfying 
$$
T\max_{\vec{\mbu}\in \prod_i\mathcal{U}^i} F(\overline{\mbx},
\vec{\mbu}) -T\min_{\mbx\in \mathcal{X}} \max_{\vec{\mbu}\in
  \prod_i\mathcal{U}^i} F(\mbx, \vec{\mbu}) \leq  R_x(T,\delta) + \max_i R^i_u(T, \delta).
$$
\end{theorem}

\begin{proof}
  For the sake of exposition  we use
  \(f(\overline{\mbx}, \vec{\mbu})\)
  to denote the vector \((f(\overline{\mbx}, \mbu_i))_i\). First
  observe that for any $\vec{\mbu}\in \prod_i\mathcal{U}^i$, we have
\begin{equation*}
\begin{split}
T \, F(\overline{\mbx}, \vec{\mbu})&=T\max_{\blambda\in \Lambda} \langle \blambda, f(\overline{\mbx}, \vec{\mbu})\rangle\\
& \leq \max_{\blambda\in \Lambda} \langle \blambda, \sum_{t=1}^T f(\mbx_t, \vec{\mbu})\rangle\\
&= \max_{\blambda\in \Lambda} \sum_{t=1}^T  \langle \blambda, f(\mbx_t, \vec{\mbu})\rangle,
\end{split}
\end{equation*}
and taking maximization of $\mbu$ over both sides, we obtain
\begin{equation}\label{equ.proof.constraintone-1}
T\max_{\vec{\mbu}\in \prod_i\mathcal{U}^i} F(\overline{\mbx}, \vec{\mbu}) \leq \max_{\vec{\mbu}\in \prod_i\mathcal{U}^i}  \max_{\blambda\in \Lambda} \sum_{t=1}^T  \langle \blambda, f(\mbx_t, \vec{\mbu})\rangle.
\end{equation}
Moreover, we have
\begin{equation}\label{equ.proof.constraintone-2}
\begin{split}
 &\max_{\vec{\mbu}\in \prod_i\mathcal{U}^i}  \max_{\blambda\in \Lambda} \sum_{t=1}^T  \langle \blambda, f(\mbx_t, \vec{\mbu})\rangle -\sum_{t=1}^T F(\mbx_t, \vec{\mbu}_t)\\
 =& \max_{\vec{\mbu}\in \prod_i\mathcal{U}^i}  \max_{\blambda\in \Lambda} \sum_{t=1}^T  \langle \blambda, f(\mbx_t, \vec{\mbu})\rangle -\sum_{t=1}^T \max_{\blambda_t \in \Lambda}\langle \blambda_t, f(\mbx_t, \vec{\mbu}_t) \rangle\\
 \leq & \max_{\blambda\in \Lambda}  \max_{\vec{\mbu}\in \prod_i\mathcal{U}^i} \sum_{t=1}^T  \langle \blambda, f(\mbx_t, \vec{\mbu})\rangle -\max_{\blambda_0 \in \Lambda}\sum_{t=1}^T \langle \blambda_0, f(\mbx_t, \vec{\mbu}_t) \rangle\\
 \leq & \max_{\blambda\in \Lambda}\left\{ \max_{\vec{\mbu}\in \prod_i\mathcal{U}^i} \sum_{t=1}^T  \langle \blambda, f(\mbx_t, \vec{\mbu})\rangle
 - \sum_{t=1}^T \langle \blambda, f(\mbx_t, \vec{\mbu}_t) \rangle \right\}\\
 =&  \max_{\blambda\in \Lambda}\left\{ \sum_{i=1}^n \lambda_i \left[ \max_{\mbu^i\in \mathcal{U}^i} \sum_{t=1}^T f^i(\mbx_t, \mbu^i ) -\sum_{t=1}^T f^i(\mbx_t, \mbu^i_t )\right]\right\}\\
 \stackrel{(a)}{\leq} &  \max_{\blambda\in \Lambda}\left\{ \sum_{i=1}^n \blambda_i \left[ R^i_u(T, \delta)\right]\right\}\\
 \leq  &  \max_{\blambda\in \Lambda}\left\{ \sum_{i=1}^n \blambda_i \left[\max_{j=1,2,\cdots, n} R^j_u(T, \delta)\right]\right\}\\
 \stackrel{(b)}{\leq} & \max_i R^i_u(T, \delta),
\end{split}
\end{equation}
holds with probability $1-n\delta$. Here, (a) holds from the
assumption that the $\mbu_i$ are learnable, and (b) holds because
$\Lambda\subseteq \Delta(n)$.

We also have
\begin{equation}\label{equ.proof.constraintone-3}
\sum_{t=1}^T F(\mbx_t, \vec{\mbu}_t) -\min_{\mbx\in \mathcal{X}} \sum_{t=1}^T F(\mbx, \vec{\mbu}_t) \leq R_x(T,\delta),
\end{equation}
and
\begin{equation}\label{equ.proof.constraintone-4}
\begin{split}
&\min_{\mbx\in \mathcal{X}} \sum_{t=1}^T F(\mbx, \vec{\mbu}_t) -T\min_{\mbx\in \mathcal{X}} \max_{\vec{\mbu}\in \prod_i\mathcal{U}^i} F(\mbx, \vec{\mbu})\\
&=\min_{\mbx\in \mathcal{X}} \sum_{t=1}^T F(\mbx, \vec{\mbu}_t) -\min_{\mbx\in \mathcal{X}} \sum_{t=1}^T \max_{\vec{\mbu}\in \prod_i\mathcal{U}^i} F(\mbx, \vec{\mbu})\\
& \leq 0.
\end{split}
\end{equation}
Summing up Equation~\eqref{equ.proof.constraintone-1},~\eqref{equ.proof.constraintone-2},~\eqref{equ.proof.constraintone-3},~\eqref{equ.proof.constraintone-4}, we get
$$T\max_{\vec{\mbu}\in \prod_i\mathcal{U}^i} F(\overline{\mbx}, \vec{\mbu}) -T\min_{\mbx\in \mathcal{X}} \max_{\vec{\mbu}\in \prod_i\mathcal{U}^i} F(\mbx, \vec{\mbu}) \leq  \max_i R^i_u(T, \delta)+ R_x(T,\delta),
$$
with probability $1-(n+1)\delta$.
\end{proof}

\begin{remark}
  We remark that Theorem~\ref{thm.constraint-app1-main} exploits the
  asymmetric nature of the primal player and the dual player, and in
  particular it is not a
  direct consequence of Theorem \ref{thm:main-result}. Indeed when
  $\mbu^i$ are outputs of learning algorithms for the $f^i(\cdot,\cdot)$,
  then the concatenation of these outputs is not necessarily a learning algorithm
  for $F(\cdot,\cdot)$: For example, let
  $\mathcal{U}^1 =\mathcal{U}^2 =[-1,+1]$ and
  $\mathcal{X} \in \{-1, +1\}^2$. Furthermore we pick functions
  $f^1(\mbx_t,\mbu^1) =\frac{1}{2}|\mbx^1_t-\mbu^1|$, and
  $f^2(\mbx_t,\mbu^2) =\frac{1}{2}|\mbx_t^2-\mbu^2|$. Now suppose
  $\{\mbx_t\}$ are an iid sequence uniformly sampled from $\{-1, +1\}^2$. It is
  easy to see that for $f^1$ and $f^2$ the choice
  $\mbu^1_t\equiv \mbu^2_t \equiv 0$ achieves zero regret.   However,
  $\mathbb{E}\sum_{t=1}^T F(\mbx_t,\vec{\mbu}_t) =\frac{1}{2}T$,
  whereas
  $\mathbb{E}\sum_{t=1}^T F(\mbx_t, [1,1]^{\top}) =\frac{3}{4}T$
  because with probability $3/4$, either $f^1(\mbx_t,1)=1$ or
  $f^2(\mbx_t,1)=1$. Thus, $\mbu_t$ is not a no-regret sequence
  for $F(\mbx_t,\cdot)$.
\end{remark}

\begin{remark} Notice that the learnability assumption in this
  approach is asymmetric: for the primal player we require $F(\cdot)$
  to be learnable, whereas for the dual players we only require
  $f^i(\cdot)$ to be learnable. We discuss some implications of this
  asymmetry. Recall that taking the maximum over $\blambda\in \Lambda$
  preserves convexity but not linearity w.r.t.\ the $\mbx$
  argument. Thus, if $f^i(\cdot)$ are linear w.r.t.\ $\mbx$ and learnable
  via {\em Follow the Perturbed Leader}, the same approach will not
  extend to $F(\cdot)$. On the other hand, if $f^i(\cdot)$ are convex
  w.r.t.\ $\mbx$ and learnable via {\em online gradient descent}, then
  $F(\cdot)$ is also learnable using {\em online gradient descent}. In
  particular, the gradient of $F(\cdot)$ is obtained as
  $\langle \blambda^*, \nabla f^i(\cdot) \rangle$ where $\blambda^*$
  maximizes $\langle \blambda, f^i(\cdot) \rangle$ over $\Lambda$.
\end{remark}

\subsection{Biased Play with a Strong Learner}
We briefly discuss the biased imaginary play case, where one player is equipped with a strong learner, and the other player is equipped with an optimization oracle.  We start with the Explicit Maximum approach.
\begin{assumption}[Strong $\mbu$ learner]
\label{ass.slearn-u-multi}
\begin{enumerate} 
\item For each  $i=1,\cdots, n$, there exists an online learning algorithm
  $\mathcal{L}^i_{u}$ for $\mbu^i$, such that for any anticipatory sequence
  $\{\mbx'_t\}_{t=1}^\infty$, the following holds with a probability
  $1-\delta$
$$ \max_{\mbu^i\in \mathcal{U}^i}    \sum_{t=1}^T f^i(\mbx'_t, \mbu^i) - \sum_{t=1}^T f^i(\mbx'_t, \mbu_t^i)  \leq R_u^i(T, \delta),
$$
where $\{\mbu^i_t\}$ is the output of $\mathcal{L}^i_{u}$.
\item Given $\vec{\mbu}$, there exists an optimization oracle for
  $\mbx$ w.r.t.\ $F(\cdot)$, that computes
$$ \mbx^*=\arg\max_{\mbx\in \mathcal{X}} F(\mbx, \mbu).
$$
 \end{enumerate}
\end{assumption}

\begin{algorithm}
  \caption{Robust Optimization via Strong Dual Learner  }
  \label{alg:sdual-multi}
  \begin{algorithmic}[1]
    \REQUIRE function \(f\), learners \(\mathcal L_u^i\)
    satisfying Assumptions~\ref{ass:setup} and~\ref{ass.slearn-u-multi}. 
    \ENSURE point $\overline{\mbx}$
    \FOR{\(t = 1, \dots, T\)}
      \STATE \(\mbu_t^i \leftarrow \mathcal L_u^i (\mbx_1,\mbu_1^i,\cdots,\mbx_{t-1}, \mbu_{t-1}^i),\,\,\ i=1,\cdots, n \)     
       \STATE \(\mbx_t \leftarrow \arg\max_{\mbx\in \mathcal{X}} F(\mbx, \vec{\mbu_t})\)
    \ENDFOR
      \STATE \(\overline{\mbx} \leftarrow \frac{1}{T}\sum_{t=1}^T \mbx_t\)

  \end{algorithmic}
\end{algorithm}

\begin{assumption}[Strong $\mbx$ learner]
\label{ass.slearn-x-multi}
\begin{enumerate} 
\item There exists an online learning algorithm
  $\mathcal{L}_{x}$ for $\mbx$, such that for any anticipatory sequence
  $\{\vec{\mbu}'_t\}_{t=1}^\infty$, the following holds with a probability
  $1-\delta$
  $$ \sum_{t=1}^T F(\mbx_t, \vec{\mbu}'_t) - \min_{\mbx\in \mathcal{X}}    \sum_{t=1}^T f^i(\mbx, , \vec{\mbu}'_t) -  \leq R_x(T, \delta),
$$
 where $\{\mbx_t\}$ is the output of $\mathcal{L}_x$.
\item Given $\mbx$, for each  $i=1,\cdots, n$, there exists an   optimization oracle for
  $\mbu^i$ w.r.t.\ $f^i(\cdot)$, that computes
$$\mbu^{i*}=\arg\max_{\mbu^i\in \mathcal{U}^i} f^i(\mbx, \mbu^i).
$$
 \end{enumerate}
\end{assumption}

\begin{algorithm}
  \caption{Robust Optimization via Strong Primal Learner  }
  \label{alg:sprimal-multi}
  \begin{algorithmic}[1]
    \REQUIRE function \(f\), learners \(\mathcal L_x\)
    satisfying Assumptions~\ref{ass:setup} and~\ref{ass.slearn-x-multi}. 
    \ENSURE point $\overline{\mbx}$
    \FOR{\(t = 1, \dots, T\)}
      \STATE \(\mbx_t \leftarrow \mathcal L_x(\mbx_1,\mbu_1^i,\cdots,\mbx_{t-1}, \mbu_{t-1}^i) \)     
       \STATE \(\mbu_t^i \leftarrow \arg\max_{\mbu^i\in \mathcal{U}^i} f^i(\mbx_t, {\mbu^i}),\,\,\ i=1,\cdots, n \)
    \ENDFOR
      \STATE \(\overline{\mbx} \leftarrow \frac{1}{T}\sum_{t=1}^T \mbx_t\)

  \end{algorithmic}
\end{algorithm}

\begin{theorem}
\begin{enumerate}
\item Suppose Assumptions~\ref{ass:setup} and~\ref{ass.slearn-u-multi} hold, then
 with probability
  $1-n\delta$, Algorithm~\ref{alg:sdual-multi} returns a point
  \(\overline{\mbx} = \frac{1}{T}\sum_{t=1}^T \mbx_t\) satisfying 
$$
T\max_{\vec{\mbu}\in \prod_i\mathcal{U}^i} F(\overline{\mbx},
\vec{\mbu}) -T\min_{\mbx\in \mathcal{X}} \max_{\vec{\mbu}\in
  \prod_i\mathcal{U}^i} F(\mbx, \vec{\mbu}) \leq    \max_i R^i_u(T, \delta).
$$
\item Suppose Assumptions~\ref{ass:setup} and~\ref{ass.slearn-x-multi} hold, then
 with probability
  $1-\delta$, Algorithm~\ref{alg:sprimal-multi} returns a point
  \(\overline{\mbx} = \frac{1}{T}\sum_{t=1}^T \mbx_t\) satisfying 
$$
T\max_{\vec{\mbu}\in \prod_i\mathcal{U}^i} F(\overline{\mbx},
\vec{\mbu}) -T\min_{\mbx\in \mathcal{X}} \max_{\vec{\mbu}\in
  \prod_i\mathcal{U}^i} F(\mbx, \vec{\mbu}) \leq  R_x(T,\delta).
$$
\end{enumerate}
\end{theorem}
\begin{proof}
The proof of first claim follows by adapting the proof of Theorem~\ref{thm.constraint-app1-main} by replacing Equation~\ref{equ.proof.constraintone-3}  by
\begin{equation*} 
\sum_{t=1}^T F(\mbx_t, \vec{\mbu}_t) -\min_{\mbx\in \mathcal{X}} \sum_{t=1}^T F(\mbx, \vec{\mbu}_t) =0,
\end{equation*}
using the fact that $\mbx_t$ is obtained via the optimization oracle. 

The proof of the second claim is by adapting the proof of Theorem~\ref{thm.constraint-app1-main}, specifically, replacing (a) of Equation~\ref{equ.proof.constraintone-2} to
$$\max_{\blambda\in \Lambda}\left\{ \sum_{i=1}^n \lambda_i \left[ \max_{\mbu^i\in \mathcal{U}^i} \sum_{t=1}^T f^i(\mbx_t, \mbu^i ) -\sum_{t=1}^T f^i(\mbx_t, \mbu^i_t )\right]\right\} = 0,$$
using the fact that $\vec{\mbu}$ is a result of the optimization oracle.
\end{proof}

Similarly, extending the distributional maximum approach to the biased imaginary play is straightforward as well. Indeed, using the fact that Problem~\ref{eq:robOptProblem-multi} can be rewritten as
$$\underset{\mbx\in \mathcal{X}}{\mbox{Minimize}}\,\, \underset{\vec{\mbu}\in \prod_i \mathcal{U}^i, \blambda\in \Lambda}{\mbox{Maximize}} g(\mbx, \vec{\mbu}, \blambda),
$$
we immediately conclude that if there exists a strong learner for $\mbx$ w.r.t.\ $g(\cdot)$, and optimization oracle both for $\mbu_i$ wrt $f^i(\cdot)$, and for $\blambda$ wrt $g(\cdot)$, which implies an optimization oracle for $(\vec{\mbu}, \blambda)$ jointly wrt $g(\cdot)$, then the problem reduces to the biased imaginary play for a single objective function discussed in Section~\ref{sec:bias}. On the other hand, suppose strong online learners exist for both $\mbu_i$ and for $\blambda$, then following a similar argument as that of Proposition~\ref{pro.multi.dual}, it can be established that playing these two learning algorithms simultaneously is a strong learner for $(\mbu_i,\blambda)$ wrt $g(\cdot)$. Therefore, Problem~\ref{eq:robOptProblem-multi} reduces to the biased imaginary play for a single objective function if an optimization oracle for $\mbx$ wrt $g(\cdot)$ exists. Notice that since  $g(\cdot)$ is linear to $\blambda$, constructing a strong learner for $\blambda$ is relatively easy (e.g., using online gradient descent).

\subsection{Randomized robust optimization}
In the single objective function case we have shown that it is
possible to relax Assumption~\ref{ass:setup}, i.e., allowing the loss
function $f(\cdot)$ and $\mathcal{X}$ to be non-convex to obtain a
randomized solution that is optimal. In general, this relaxation cannot be
extended to the multi-function case. To see that, take the
approach via the distributional maximum as an example. Directly adapting
the results from Section~\ref{sec:randomized-case} leads to an algorithm which
outputs a distribution $\mu$ such that
$$\max_{\blambda, \vec{\mbu}} E_{\mbx\sim \mu} g(\mbx, \vec{\mbu},\blambda) - \min_{\mu^*}\max_{\blambda^*, \vec{\mbu}^*} E_{\mbx^*\sim \mu^*} g(\mbx^*, \vec{\mbu}^*,\blambda^*) \rightarrow 0. 
$$ 
However, the quantity of interest that we want to compare to is
$\max_{ \vec{\mbu}} E_{\mbx\sim \mu} \big[\max_{\blambda} g(\mbx,
\vec{\mbu},\blambda)\big]$: if the algorithm outputs a solution $\mbx$ randomly, then the
corresponding worst-case $\blambda$ should adapt to the random choice
of $\mbx$. 

Here, we consider the following special case, for which extension to
multiple functions is possible. Specifically, we consider the case
that $f_i(\mbx, \mbu^i) \geq 0$ for all $i, \mbx, \mbu^i$, and seek
a solution $\mbx^*$ such that $\max_i\max_{\mbu^i}
f_i(\mbx^*,\mbu^i)=0$.
This is motivated by the robust feasibility problem where
$f_i(\mbx, \mbu^i)$ is the violation of the $i$-th constraint under
parameter realization $\mbu^i$. In particular, we assume $\Lambda=\Delta(n)$.

We will show that in this case, both explicit maximum and distributional maximum approach can be
 extended. The algorithms are identical to
 Algorithms~\ref{alg:rool-const} and 
 and \ref{alg:rool-const-learnlambda}, except for the last stage,
 where we output a distribution $\mu$ which is the empirical distribution of
 $\mbx_t$ instead of the average. We call the resulting algorithm {\em randomized
   explicit maximum} and {\em randomized distributional maximum}
 respectively.
 
\begin{theorem}\label{thm.constraint-app1-main-randomized} Suppose Assumption~\ref{ass:learn_const1} holds,
with probability
  $1-(n+1)\delta$, the {\em randomized explicit maximum} algorithm returns a distribution $\mu$
   satisfying 
$$
 \max_{\vec{\mbu}\in \prod_i\mathcal{U}^i}\mathbf{E}_{\mbx\sim \mu} F(\mbx,
\vec{\mbu}) -(n+1)\min_{\mu'\in \Delta(\mathcal{X})} \max_{\vec{\mbu}'\in
  \prod_i\mathcal{U}^i}\mathbf{E}_{\mbx'\sim \mu'}   F(\mbx', \vec{\mbu}') \leq \frac{(n+1)R_x(T,\delta) + \sum_{i=1}^n R^i_u(T, \delta)}{T} .
$$
Moreover, if the problem is infeasible, i.e., there exists $\varepsilon > 0$, such that for any $\mu'\in \Delta(\mathcal{X})$,
$$\max_{\vec{\mbu}'\in
  \prod_i\mathcal{U}^i}\mathbf{E}_{\mbx'\sim \mu'}   F(\mbx', \vec{\mbu}')>\epsilon,
$$
then with probability $1-n\delta$,
$$\liminf_{T} \frac{1}{T} \sum_{t=1}^T F(\mbx_t, \vec{\mbu}_t) \geq \frac{1}{n+1} \epsilon.
$$
\end{theorem}

\begin{proof}

For any $\vec{\mbu}$,
\begin{equation*}
\begin{split}
& \sum_{t=1}^T F(\mbx_t, \vec{\mbu}) -\sum_{t=1}^T F(\mbx_t, \vec{\mbu}_t)\\
=& \sum_{t=1}^T  \max_{i'_t\in \{1,2,\cdots, n\}} f^{i'_t}(\mbx_t, \mbu^{i'_t}) - \sum_{t=1}^T \max_{i_t\in \{1,2,\cdots, n\}} f^{i_t}(\mbx_t, \mbu^{i_t}_t) \\
\leq &  \sum_{t=T}^t \max_{i_t\in \{1,2,\cdots, n\}} \left[ f^{i_t}(\mbx_t, \mbu^{i_t}) - f^{i_t}(\mbx_t, \mbu^{i_t}_t)\right].
\end{split}
\end{equation*}
Denote by $i_t^*\triangleq \arg\max_{i_t\in \{1,2,\cdots, n\}} \left[ f^{i_t}(\mbx_t, \mbu^{i_t}) - f^{i_t}(\mbx_t, \mbu^{i_t}_t)\right]$, then the right hand side equals
\begin{equation*}
\begin{split}
&\sum_{t=1}^T   \left[ f^{i_t^*}(\mbx_t, \mbu^{i_t}) - f^{i_t^*}(\mbx_t, \mbu^{i_t}_t)\right]\\
=&\sum_{t=1}^T \sum_{i=1}^n  \left[ f^{i}(\mbx_t, \mbu^{i}) - f^{i}(\mbx_t, \mbu^{i}_t)\right] - \sum_{t=1}^T \sum_{i\not=i_t^*}  \left[ f^{i}(\mbx_t, \mbu^{i}) - f^{i}(\mbx_t, \mbu^{i}_t)\right] \\
\leq & \sum_{t=1}^T \sum_{i=1}^n  \left[ f^{i}(\mbx_t, \mbu^{i}) - f^{i}(\mbx_t, \mbu^{i}_t)\right] - \min_{\vec{\mbu}'}\sum_{t=1}^T \sum_{i\not=i_t^*}  \left[ f^{i}(\mbx_t, \mbu^{i'}) - f^{i}(\mbx_t, \mbu^{i}_t)\right]\\
\leq &  \sum_{i=1}^n \sum_{t=1}^T \left[ f^{i}(\mbx_t, \mbu^{i}) - f^{i}(\mbx_t, \mbu^{i}_t)\right] +\sum_{i=1}^n \sum_{t=1}^T f^{i}(\mbx_t, \mbu^{i}_t),
\end{split}
\end{equation*}
where we use non-negativity of $f(\cdot)$ for the last inequality. Take maximum over $\vec{\mbu}$, we have
\begin{eqnarray}\label{equ.usedforinfe}
&&\max_{\vec{\mbu}}\sum_{t=1}^T F(\mbx_t, \vec{\mbu}) -\sum_{t=1}^T F(\mbx_t, \vec{\mbu}_t) \nonumber\\
&\leq & \max_{\vec{\mbu}} \sum_{i=1}^n \sum_{t=1}^T \left[ f^{i}(\mbx_t, \mbu^{i}) - f^{i}(\mbx_t, \mbu^{i}_t)\right] + \sum_{i=1}^n \sum_{t=1}^T f^{i}(\mbx_t, \mbu^{i}_t) \nonumber \\
&\leq & \sum_{i=1}^n R^i_u(T, \delta) + n\sum_{t=1}^T F(\mbx_t, \vec{\mbu}_t)
\end{eqnarray}

holds with probability $1-n\delta$. Here, the last equality holds by 
Assumption~\ref{ass:learn_const1} requiring that the $\mbu_i$ are learnable.

We also have
\begin{equation*}
\sum_{t=1}^T F(\mbx_t, \vec{\mbu}_t) -\min_{\mbx\in \mathcal{X}} \sum_{t=1}^T F(\mbx, \vec{\mbu}_t) \leq R_x(T,\delta).
\end{equation*}

Summing up the two inequalities, we obtain
\begin{eqnarray*}
\max_{\vec{\mbu}}\sum_{t=1}^T F(\mbx_t, \vec{\mbu}) -\min_{\mbx\in \mathcal{X}} \sum_{t=1}^T F(\mbx, \vec{\mbu}_t) &\leq& \sum_{i=1}^n R^i_u(T, \delta) + n\sum_{t=1}^T F(\mbx_t, \vec{\mbu}_t) +R_x(T,\delta)\\
&\leq& \sum_{i=1}^n R^i_u(T, \delta) + n\min_{\mbx\in \mathcal{X}} \sum_{t=1}^T F(\mbx, \vec{\mbu}_t)  +(n+1)R_x(T,\delta),
\end{eqnarray*}
which leads to
$$\max_{\vec{\mbu}}\sum_{t=1}^T F(\mbx_t, \vec{\mbu}) -(n+1)\min_{\mbx\in \mathcal{X}} \sum_{t=1}^T F(\mbx, \vec{\mbu}_t) \leq
\sum_{i=1}^n R^i_u(T, \delta)   +(n+1)R_x(T,\delta),
$$
The rest of the proof for the first claim follows similarly as that of Theorem~\ref{thm:main-result-distribution}.

To establish the second statement, notice that when the problem is infeasible, by definition there exists $\epsilon>0$ such that
$$\frac{1}{T}\sum_{t=1}^T F(\mbx_t, \vec{\mbu})\geq \epsilon.
$$
Combining this with Equation~\eqref{equ.usedforinfe} we have with probability $1-n\delta$,
 $$\liminf_T\frac{1}{T} \sum_{t=1}^T F(\mbx_t, \vec{\mbu}_t) \geq \frac{\epsilon}{n+1}. 
 $$
\end{proof}

 \begin{theorem} Suppose that Assumption ~\ref{assu.lambdalearn} holds
   and that \(\Lambda = \Delta(n)\).
   Then with probability $1-(n+2)\delta$, the randomized distributional
   maximum algorithm returns a distribution $\mu$ satisfying
\begin{multline}
\max_{\vec{\mbu}\in \prod_i \mathcal{U}^i} \mathbb{E}_{\mbx\sim
  \mu} \max_{ \blambda \in \Lambda} g(\mbx, \vec{\mbu}, \blambda)-n
\min_{\mu'\in \Delta(\mathcal{X})}   \max_{\vec{\mbu}'\in \prod_i
  \mathcal{U}^i} \mathbb{E}_{\mbx'\sim \mu'} \max_{\blambda'\in
  \Lambda} g(\mbx', \vec{\mbu}', \blambda') \\ \le \frac{n\left\{ R_x(T,\delta) + \max_i R^i_u(T, \delta) +R_{\lambda}(T, \delta)\right\}}{T}.
\end{multline}
Moreover, \(\mu\) is the empirical distribution over the \(\mbx_t \in
\mathcal X\) played by the \(\mbx\)-player.
\end{theorem}
\begin{proof}
By Proposition~\ref{pro.multi.dual} we have 
 $$\max_{\blambda\in \Lambda, \vec{\mbu}\in \prod_i\mathcal{U}^i} \sum_{t=1}^T g(\mbx_t, \vec{\mbu},  \blambda)
  -  \sum_{t=1}^T g(\mbx_t, \vec{\mbu}_t,  \blambda_t)
   \leq \max_i  R^i_{u}(T,\delta)  +R_{\lambda}(T, \delta),
 $$
 and by Assumption~\ref{assu.lambdalearn} (iii) we obtain a
 distribution \(\mu\)
$$\sum_{t=1}^T  g(\mbx_t, \vec{\mbu}_t,  \blambda_t) - \min_{\mbx \in \mathcal X} \sum_{t=1}^T g(\mbx, \vec{\mbu}_t,  \blambda_t) \leq R_{x}(T,\delta).
$$
Following the proof of Theorem~\ref{thm:main-result-distribution}, we have
$$ \max_{\vec{\mbu}\in \prod_i \mathcal{U}^i}\max_{ \blambda \in
  \Lambda} \mathbb{E}_{\mbx\sim \mu}  g(\mbx, \vec{\mbu}, \blambda)-
\min_{\mu'\in \Delta(\mathcal{X})}   \max_{\vec{\mbu}'\in \prod_i
  \mathcal{U}^i}\max_{\blambda'\in \Lambda} \mathbb{E}_{\mbx'\sim
  \mu'}  g(\mbx', \vec{\mbu}', \blambda')\le \frac{R_x(T,\delta) + \max_i
R^i_u(T, \delta) +R_{\lambda}(T, \delta)}{T}.
$$
Now by the non-negativity assumption of the $f_i(\cdot,\cdot)$ (and
hence $g(\cdot)$), we have
$$ \max_{ \blambda \in \Lambda} g(\mbx, \vec{\mbu}, \blambda) \leq
\sum_{i=1}^n g(\mbx, \vec{\mbu}, \mbe_i)
$$
where $\mbe_i$ is the $i$-th basis vector (i.e., the $i$-th entry equals \(1\), and the rest equals \(0\)). This leads to
\begin{eqnarray*}
  && \max_{\vec{\mbu}\in \prod_i \mathcal{U}^i} \mathbb{E}_{\mbx\sim \mu} \max_{ \blambda \in \Lambda} g(\mbx, \vec{\mbu}, \blambda)\\
  &\leq & \max_{\vec{\mbu}\in \prod_i \mathcal{U}^i} \mathbb{E}_{\mbx\sim \mu}  \sum_{i=1}^n g(\mbx, \vec{\mbu}, \mbe_i)\\
  &\stackrel{(a)}{\leq} &  \max_{\vec{\mbu}\in \prod_i \mathcal{U}^i} \left[ n\cdot \max_{ \blambda \in \Lambda} \mathbb{E}_{\mbx\sim \mu}  g(\mbx, \vec{\mbu}, \blambda) \right] ,
\end{eqnarray*}
where (a) holds because $\Lambda=\Delta(n)$. Using
$$\min_{\mu'\in \Delta(\mathcal{X})}   \max_{\vec{\mbu}'\in \prod_i
  \mathcal{U}^i}\max_{\blambda'\in \Lambda} \mathbb{E}_{\mbx'\sim
  \mu'}  g(\mbx', \vec{\mbu}', \blambda') \le \min_{\mu'\in
  \Delta(\mathcal{X})}   \max_{\vec{\mbu}'\in \prod_i \mathcal{U}^i}
\mathbb{E}_{\mbx'\sim \mu'} \max_{\blambda'\in \Lambda} g(\mbx',
\vec{\mbu}', \blambda'),
$$
the theorem follows via
\begin{align}
&\ \max_{\vec{\mbu}\in \prod_i \mathcal{U}^i} \mathbb{E}_{\mbx\sim \mu}
  \max_{ \blambda \in \Lambda} g(\mbx, \vec{\mbu}, \blambda) -
  n \min_{\mu'\in
  \Delta(\mathcal{X})}   \max_{\vec{\mbu}'\in \prod_i \mathcal{U}^i}
\mathbb{E}_{\mbx'\sim \mu'} \max_{\blambda'\in \Lambda} g(\mbx',
\vec{\mbu}', \blambda').  \\
\leq &\ \max_{\vec{\mbu}\in \prod_i \mathcal{U}^i} \mathbb{E}_{\mbx\sim \mu}
  \max_{ \blambda \in \Lambda} g(\mbx, \vec{\mbu}, \blambda) -
  n \min_{\mu'\in \Delta(\mathcal{X})}   \max_{\vec{\mbu}'\in \prod_i
  \mathcal{U}^i}\max_{\blambda'\in \Lambda} \mathbb{E}_{\mbx'\sim
  \mu'}  g(\mbx', \vec{\mbu}', \blambda') \\
 \leq &\  n \max_{\vec{\mbu}\in \prod_i \mathcal{U}^i}\max_{ \blambda \in
  \Lambda} \mathbb{E}_{\mbx\sim \mu}  g(\mbx, \vec{\mbu}, \blambda) -
  n \min_{\mu'\in \Delta(\mathcal{X})}   \max_{\vec{\mbu}'\in \prod_i
  \mathcal{U}^i}\max_{\blambda'\in \Lambda} \mathbb{E}_{\mbx'\sim
  \mu'}  g(\mbx', \vec{\mbu}', \blambda') \\ \le &\  n \left(\frac{R_x(T,\delta) + \max_i
R^i_u(T, \delta) +R_{\lambda}(T, \delta)}{T}\right).
\end{align}
\end{proof}

The above result bounds the expected performance of a random 
$\mbx$ according to $\mu$ versus $n$ times the performance of the
optimal $\mu'$; this notion is very close to the concept of
\(\alpha\)-regret. However, in our context this approximation does not
matter: there exists a robust feasible $\mbx^*$ if
and only if
\[\min_{\mu'\in \Delta(\mathcal{X})}   \max_{\vec{\mbu}'\in \prod_i
  \mathcal{U}^i} \mathbb{E}_{\mbx'\sim \mu'} \max_{\blambda'\in
  \Lambda} g(\mbx', \vec{\mbu}', \blambda') = 0.\]
Thus, the multiplicative factor
has no impact in the feasibility case, if the problem is feasible,
then we are able to output a randomized solution whose constraint
violation is sublinear in $T$. Note however, that this will likely
reduce the convergence rate by some factor, which is slower but still
sublinear in \(T\) for reasonable learners.

\section{Application: RO with IP-representable sets}
\label{sec:robust-optim-with}

In this section we provide an  application, which is naturally captured by Problem~\eqref{eq:robOptProblem-expected}.
In specific, we instantiate our framework for the case where both
the feasible region \(\mathcal X\)
as well as the uncertainty set \(\mathcal U\)
can be represented via an (Mixed-) Integer Programming Problem. In
many cases optimizing over such sets is NP-hard and in particular no
good characterization in terms of valid inequalities can be provided
(unless \(\text{coNP} = \text{NP}\)),
however even very large instances can be solved in reasonable time
using state-of-the-art integer programming solvers such as e.g.,
\texttt{CPLEX} or \texttt{Gurobi}. In these cases the assumption that
we have a linear programming oracle for these sets is justified and we
assume in this section that we have access to both the feasible region
as well as the uncertainty set by means of a linear optimization oracle.

We assume that \(f\) is linear in \(\mbx\) and \(\mbu\), i.e., we let
\(f(\mbx,\mbu) = \mbu^\intercal A \mbx\) for some matrix \(A\) of
appropriate dimension. For this setup both learners can use the
\emph{Follow the Perturbed Leader} algorithm
\citep[see][]{kalai2005efficient} that works for online learning of
\emph{linear functions} over arbitrary (non-convex) sets as long as we
have access to a linear programming oracle for the set. The FPL
algorithm has the following high-probability regret guarantee (Corollary~\ref{cor:hpFPL}), which
follows almost immediately from \citep{neu2013efficient}.

\begin{theorem}[\cite{neu2013efficient}]
\label{thm:fpl}
Let \(S \subseteq \R_+^d\)
with \(\norm[1]{s} \leq m\)
for all \(s \in S\).
Assume that the absolute values of all losses are bounded by \(1\).
Then FPL over \(S\)
achieves an expected regret of \(O(m^{3/2} \sqrt{T \log d})\).
\end{theorem}

From the above, via standard arguments, we can easily obtain a
high-probability version of the regret bound of FPL:

\begin{corollary}[High-probability FPL]
\label{cor:hpFPL}
Let \(\delta >0\)
and let \(S \subseteq \R^d\)
with \(\norm[1]{s} \leq m\)
for all \(s \in S\).
Assume that the absolute values of all losses are bounded by \(1\).
Then with probability at least \(1-\delta\), algorithm
FPL over \(S\) achieves regret \(O(m^{3/2} \sqrt{T \log d + \log
  \frac{1}{\delta}})\).\end{corollary}
\begin{proof}
  For any sequence of losses
  \(\{\ell_t\}_{t\in [T]}\) with \(|\ell_t s| \leq 1\) for all \(s \in
  S, t \in [T]\) by Theorem~\ref{thm:fpl} FPL produces a sequence of actions
  \(\{s_t\}_t \subseteq S\) with
\[\expectation(t){\sum_{t \in [T]} \ell_t s_t} - \min_{s \in S} \sum_{t
  \in [T]} \ell_t s \leq O(m^{3/2} \sqrt{T \log d}),\]
where \(\expectation(t){\cdot}\)
is the conditional expectation conditioned on FPL's internal
randomization up to round \(t-1\).
The random variables
\(Z_t = \sum_{t \in [T]} \ell_t s_t - \expectation(t){\sum_{t \in [T]}
  \ell_t s_t}\) form a martingale difference sequence with respect to
FPL's internal randomization for \(t \in [T]\) and we have
\(Z_t \leq 2\). Therefore with probability at least
\(1-\delta\), we obtain:
\[\sum_{t \in [T]} \ell_t s_t - \expectation(t){\sum_{t \in [T]}
  \ell_t s_t} \leq 2 \sqrt{T \log \frac{1}{\delta}},\]
via Azuma's inequality and hence the claim follows by summing up both
inequalities. 
\end{proof}

We would like to remark that FPL also admits an approximate variant;
see \citet{ben2015oracle}. 

\paragraph{Setup.} \vspace{-0.5em}
\begin{enumerate}
\item \emph{Learner \(\mathcal L_x\):} FPL over \(\mathcal X\) and access
  to a linear optimization oracle for \(\mathcal X\). \vspace{-0.5em}
\item \emph{Learner \(\mathcal L_y\):} FPL over \(\mathcal U\) and access to a linear optimization oracle
for \(\mathcal U\). \vspace{-0.5em}
\item \emph{Function \(f\):} Require that \(f(\cdot,\mbu)\) is linear for all
\(u \in \mathcal U\) and \(f(\mbx,\cdot)\) is linear for all \(\mbx
\in \mathcal X\), i.e.,
\(f(\mbx,\mbu) = \mbu^\intercal A \mbx\). \vspace{-0.5em}
\item \emph{Feedback:} Player \(\mbx\) observes \(\mbu^\intercal A\) and player
  \(\mbu\) observes \(A \mbx\). \vspace{-0.5em}
\end{enumerate}

Combining the above with Theorem~\ref{thm:main-result} we obtain: 

\begin{corollary}
\label{cor:robIPCor}
  Let the setup be given as above.  With probability $1-2\delta$,
  Algorithm~\ref{alg:r2ool} returns an empirical distribution 
  $\overline{\mu}$ satisfying
\begin{equation*}\begin{split}
  &\max_{\nu'\in \Delta(\mathcal{U})} \underset{\mbx\sim \overline{\mu},\mbu\sim \nu'}{\mathbb{E} } f(\mbx, \mbu) -\min_{\mu^* \in \Delta(\mathcal{X}) } \max_{\nu\in \Delta(\mathcal{U})} \underset{\mbx\sim \mu^*, \mbu\sim \nu}{ \mathbb{E}}  f(\mbx, \mbu)\\
& \leq O\left(\frac{m_{\mathcal X}^{3/2} \sqrt{\log d_{\mathcal X} + \log
      \frac{1}{\delta}} + m_{\mathcal U} ^{3/2} \sqrt{\log d_{\mathcal U} + \log
      \frac{1}{\delta}}}{\sqrt{T}}\right),
\end{split}\end{equation*}
where \(m_\mathcal{X}, d_{\mathcal X}, m_{\mathcal U}, d_{\mathcal
  U}\) are the dimensions and \(\ell_1\)-diameters of \(\mathcal X\)
and \(\mathcal U\) respectively. 
\end{corollary}

\begin{remark}[Infeasibility of dualization approach]
Note that this setup is a good example where we cannot solve the
robust problem via dualizing the uncertainty set. If the uncertainty set
\(\mathcal U\) is intractable (e.g., optimizing over \(\mathcal U\)
is NP-hard or \(\conv{\mathcal U}\) has high extension complexity),
then a tractable dual formulation cannot exist. However, having
access to linear optimization oracles still allows us to solve the
robust optimization problem. 
\end{remark}

\begin{remark}[Non-adaptivity of \(\mbu\)]
  We would like to stress, that we \emph{do not} solve \(\min_{\mbx \in
    \mathcal X} \max_{\mbu \in \mathcal U} f(\mbx,\mbu)\) with
  solution pair \((\mbx,\mbu) \in \mathcal X \times \mathcal U\) but rather
  \(\bar \mbx \in \conv{\mathcal X}\), i.e., it is a \emph{mixed
    strategy} and \(\min_{\mbx \in
    \mathcal X} \max_{\mbu \in \mathcal U} f(\mbx,\mbu)\) should be
  considered the value that we benchmark against. We showed that \(\bar \mbx\) can achieve the same
  value (with vanishing regret) as compared to choosing the minimal \(\mbx^*\), i.e., 
$$\max_{\mbu'\in \mathcal{U}} f(\overline{\mbx}, \mbu') \leq 
\min_{\mbx^*\in \mathcal{X}} \max_{\mbu\in \mathcal{U}} f(\mbx^*,
\mbu) + o(1),
$$
as \(\mbu\)
is not adaptive with respect to \(\mbx\).
Here non-adaptivity refers to the \(\mbu\)-player
having to make her decision \(\mbu_t\)
in round \(t\)
without knowing what the \(\mbx\)-player
will play as \(\mbx_t\) in round \(t\) (and also vice versa with swapped roles).
Put differently, the non-adaptivity ensures that we solve the problem
\begin{equation*}
\begin{split}
&\min_{\mu \in \Delta(\mathcal X)} \max_{\kappa \in \Delta(\mathcal{U})}
\expectation(\mbx \sim \mu, \mbu \sim \kappa){f(\mbx,
\mbu)}\\
 = & \min_{\mu \in \Delta(\mathcal X)} \max_{\mbu \in  \mathcal{U}}
\expectation(\mbx \sim \mu){f(\mbx,
\mbu)},\end{split}\end{equation*}
which is equivalent to solving the problem over the respective convex
hulls of \(\mathcal X\) and \(\mathcal U\). 
\end{remark}

We now present a sample application. 

\begin{example}[Robust Routing in Unrealiable Networks]
\label{exa:robustRouting}
  Let \(G = (V,E)\)
  be an undirected graph with distinguished source \(s\)
  and sink \(t\)
  with \(s,t \in V\).
  We consider the setup of robust minimal cost routing in \(G\)
  with unrealiable edges: we want to find a route (without revisiting edge) of minimum cost from \(s\)
  to \(t\)
  in \(G\)
  under various scenarios where edges might become unavailable. 

  We let the \(\mbx\)
  learner choose paths over \(G(V,E)\),
  which can be represented as
  \(\mathcal X \subseteq \{0,1\}^{E}\).
  The corresponding optimization problem over \(\mathcal X\)
  for nonnegative costs \(\{c_e\}_{e \in E}\)
  can be solved e.g., using network flows or shortest path
  computations. The \(\mbu\)
  learner is playing \emph{edge removal} subject to
    budget and connectivity constraints: The set \(\mathcal U\) is
    given by
\[\mathcal U \coloneqq \{ C \subseteq E \mid \size{C} \leq K,
G[E\setminus C] \text{ is connected}\}, \]
 i.e., the adversary can remove a small number of edges as long as the graph
remains connected, e.g., removing valuable shortcuts between nodes but
the adversary cannot adapt to the chosen route. 

It remains to specify the cost. We have the matrix \(A \in
\R_+^{E \times E}\) with entries \(A = M \cdot I\), where \(I\) is the
\(E \times E\) identity matrix and we use cost \(c \mbx +
\mbu^\intercal A \mbx\). In standard fashion this can be reformulated
as a cost matrix \(\tilde A = \left(
  \begin{matrix}
    c \\
    A
  \end{matrix}
\right)\) and we prepend the \(\mbu\)-actions with a \(1\), i.e.,
\(\tilde \mbu = (1,\mbu)\). We further assume that \(M \gg \max_{e \in
E} c_e\). By Corollary~\ref{cor:robIPCor} we obtain that after \(T
\geq O(1) \cdot \frac{M^2 (\size{V}^{3} + \size{K}^3) (\log \size{E} + \log
  \frac{1}{2 \delta})}{\varepsilon^2}\)
rounds with probability at least \(1-\delta\) the solution
\(\bar x\) satisfies:
$$\max_{\mbu'\in \mathcal{U}} f(\overline{\mbx}, \mbu') \leq
\min_{\mbx^*\in \mathcal{X}} \max_{\mbu\in \mathcal{U}} f(\mbx^*,
\mbu) + \varepsilon,
$$
and since \(\bar \mbx\) is a distribution over routes in the graph
\(G\) and \(\mbu\) is non-adaptive, we can implement the
solution by sampling from \(\mbx\) and playing the routes. 
\end{example}

Similar examples can be readily obtained, e.g., for spanning
trees, matchings, and permutahedra.

\section{Application: (Robust) MDPs} 
\label{sec:applications}

Another application of our framework
is solving robust MDPs with {\em
  non-rectangular} uncertainty sets, where {\em only the reward
  parameters} are subject to uncertainty; other applications have been
relegated to Supplementary Material~\ref{sec:robust-optim-with}. An MDP is defined by a
6-tuple:
$\langle \mathcal{T}, \gamma, \mathcal{S}, \mathcal{A}, \mathbf{p},
\mathbf{r} \rangle $, where $\mathcal{T}$ is the (possibly infinite)
decision horizon, $\gamma\in (0, 1]$ is the discount factor,
$\mathcal{S}$ is the state space, $\mathcal{A}$ is the action space,
and $\mathbf{p}$ and $\mathbf{r}$ are the transition probability and
reward respectively.  The decision criterion is to find a policy $\pi$ that
maximizes the expected cumulative discounted reward. See the
classical textbook for more
details~\cite{Puterman94}.  

Robust MDPs~\cite{Nilim04,Iyengar03} are concerned with solving MDPs
under {\em parametric uncertainty}: here $\bp$ and $\br$ are
unknown, but instead a so-called ``uncertainty set'' $\mathcal{U}$ is
given such that $(\bp, \br)\in \mathcal{U}$, and the decision
criterion is to find an optimal policy $\pi$ for the worst-case
parameter realization in $\mathcal{U}$:
 \begin{equation}\label{equ.rmdp.generalform}\underset{\pi}{\mbox{Maximize}}\,\, \min_{(\bp, \br)\in \mathcal{U}} \big\{\mathbb{E}_{\pi, \bp, \br}\big[ \sum_{t=1}^\mathcal{T} \gamma^{t-1} r(\tilde{s}_t, \tilde{a}_t) \big]\big\}.
 \end{equation}
 This problem is in general hard even when $\mathcal{U}$ is a
 relatively simple set~\cite{wiesemann2013robust}.  The two special cases where the problem
 can be solved efficiently are (1) when $\mathcal{U}$ is {\em rectangular},
 i.e., it is a Cartesian product of uncertainty sets of parameters
 of each state; or (2) only the reward parameters $\br$ are
 subject to uncertainty.  
 
 Specifically, when $\bp$ is known and only $\br$ is subject to
 uncertainty with $\br\in\mathcal{U}$,
 Problem~\eqref{equ.rmdp.generalform} can be reformulated
 as\begin{footnote}{for simplicity we consider $\mathcal{T}=\infty$,
     as the finite horizon case is easily converted into an infinite
     horizon case}\end{footnote}
 \begin{align}
   \label{equ.robustMDPinLP}
   \underset{\mathbf{x}}{\max}\,& \min_{\br\in \mathcal{U}} \sum_{s\in \mathcal{S}} \sum_{a\in \mathcal{A}} r(s,a)x(s,a)\\
\mbox{s.t.}\,  & \sum_{a\in \mathcal{A}} x(s',a) -\sum_{s\in \mathcal{S}} \sum_{a\in \mathcal{A}} \gamma p(s'|s, a) x(s,a) =\alpha(s'), \forall s',\\
 & x(s,a)\geq 0, \forall s, \forall a.
 \end{align}
 Here $\alpha(\cdot)$ is the distribution of the initial state. Let
 $\mathbf{x}^*$ be the optimal solution to the above problem, then the
 optimal policy to \eqref{equ.rmdp.generalform} is obtained by
 $q_s(a)=x^*(s,a)/(\sum_{a'\in \mathcal{A}} x^*(s, a'))$ where
 $q_s(a)$ stands for the probability of choosing action $a$ at state
 $s$.

 Notice that when $\mathcal{S}$ and $\mathcal{A}$ are even moderately
 large, solving~\eqref{equ.robustMDPinLP} can be computationally
 expensive, even when $\mathcal{U}$ is a nice and simple set. For
 example, suppose $\mathcal{U}=\{ \br \mid \|\br-\br_0\|_2\leq c \}$ for
 given $\br_0$ and $c$, i.e., the set of all reward vectors that are
 close (in the sense of Euclidean distance) to a ``nominal parameter''
 $\br_0$, then Problem~\eqref{equ.robustMDPinLP} is a second order
 cone programming with $|\mathcal{S}|\cdot |\mathcal{A}|$ linear
 constraints.
 
 We now discuss how to solve Problem~\eqref{equ.robustMDPinLP} using
 our proposed framework. Observe that
 Problem~\eqref{equ.robustMDPinLP} can be rewritten as
$$\underset{\mathbf{x}\in \mathcal{X}}{\mbox{Maximize}}\quad \underset{\br\in \mathcal{U}}{\mbox{Minimize}}\,\,\,  \mathbf{r}^\top \mathbf{x}
$$
where $\mathcal{X}$ is the feasible set of the constraints of
\eqref{equ.robustMDPinLP}. Thus, we can apply Algorithm~\ref{alg:rool}
to solve \eqref{equ.robustMDPinLP}. The key observation is that the
objective function is linear with respect to $\mathbf{x}$, which means
FPL is a weak learner for the primal player. Furthermore, each
iteration of FPL solves
$\underset{\mathbf{x}\in \mathcal{X}}{\mbox{Maximize}}\,\,
\hat{\mathbf{r}}^\top \mathbf{x} $ for some given
$\hat{\mathbf{r}}$. In the MDP context, we can solve such a problem by
first finding the optimal deterministic policy $\pi^*$ of the MDP
where the reward vector is $\hat{\mathbf{r}}$ (by value iteration or
policy iteration), and then find $\mathbf{x}^*$ corresponding to
$\pi^*$. In particular, one can obtain $\mathbf{x}^*$ via first
computing
$\mathbf{y}^* \triangleq (1-\gamma P_{\pi^*})^{-1} \alpha
=\sum_{i=1}^\infty \gamma^{i-1} P_{\pi^*}^{i-1} \alpha,$ where
$P_{\pi^*} \in \mathbb{R}^{|\mathcal{S}|\times |\mathcal{S}|}$ is the
transition matrix of the Markov chain induced by $\pi^*$, and then set
$$x^*(s,a) =\left\{\begin{array}{ll} y^*(s,a) & \mbox{if }a=\pi^*(s); \\
                                                      0         &\mbox{otherwise}.
                                                      \end{array}\right.  
$$
Thus, the computation required for obtaining an $\epsilon$-optimal
solution is
$\mathcal{O}\big( (|\mathcal{S}|^2 + |\mathcal{S}||\mathcal{A}|)\log
1/\epsilon \big)$ by combination with the computational requirement
for value iteration and for finding $\mathbf{y}^*$.  It is also
worthwhile to mention that $\|\mathbf{x}\|_1 =\frac{1}{1-\gamma}$,
which is independent of $|\mathcal{S}|$ and $|\mathcal{A}|$.

For the dual player, online gradient descent is a learner (see e.g., \cite{ocoBook}):
\begin{theorem}\label{thm.oco}
  Let $f_1(\cdot),f_2(\cdot),\cdots, f_n(\cdot),\cdots$ be a sequence
  of arbitrary convex loss functions, possibly anticipatory, and let
  $z_t$ be the output at $t$-stage of online gradient descent, then
$$ \sum_{t=1}^T f_t(z_t) -\min_{z\in \mathcal{Z}} \sum_{t=1}^T f_t(z) \leq \frac{3}{2}GD\sqrt{T}, 
$$ 
where $G$ is an upper bound of the Lipschitz continuities of
$f_t(\cdot)$ and $D$ is the diameter of $\mathcal{Z}$.
\end{theorem}
Notice that for the dual player,
$G\leq \max \|\mathbf{x}\|_2\leq \|\mathbf{x}\|_1\leq
\frac{1}{1-\gamma}$. Also notice that computing the gradient for the
dual player is trivial as it is just $\mathbf{x}_t$.

Invoking Theorem~\ref{thm:main-result}, together with
Corollary~\ref{cor:hpFPL}, and Theorem~\ref{thm.oco}, we obtain the
following. We remark that the regret bound is almost independent of
the dimensionality $|\mathcal{S}||\mathcal{A}|$.
\begin{corollary}
  Let the setup be given as above.  With probability $1-2\delta$,
  Algorithm~\ref{alg:rool} returns a point
  $\overline{\mbx}=\frac{1}{T}\sum_{t=1}^T \mbx_t$ satisfying
\begin{equation*}
\begin{split}
&\max_{\mbu'\in \mathcal{U}} f(\overline{\mbx}, \mbu') -\min_{\mbx^*\in \mathcal{X}} \max_{\mbu\in \mathcal{U}} f(\mbx^*, \mbu) \\
& \leq O\left(\left[\frac{  \sqrt{\log (|\mathcal{S}|\cdot|\mathcal{A}|) + \log
  \frac{1}{\delta}}}{(1-\gamma)^{3/2}\sqrt{T}}+ \frac{ D}{(1-\gamma)\sqrt{T}}\right]M \right),
\end{split}\end{equation*}
where $D$ is the diameter of $\mathcal{U}$, and $M=\max_{\mathbf{x}, \mathbf{r}}|\mathbf{r}^\top\mathbf{x}| \leq \max_{\mathbf{r}}|\mathbf{r}|_{\infty}/(1-\gamma).$
\end{corollary} 
\begin{remark}
  In many cases $D$ is relatively small. For example, when
  $\mathcal{U}=\{ \br \mid \|\br-\br_0\|_2\leq c \}$, then $D\leq
  c$. Another interesting case is when the perturbations are sparse,
  e.g.,
  $\mathcal{U}=\{ \br \mid \|\br-\br_0\|_{\infty}\leq c, \|\br-\br_0\|_0
  \leq d \}$. That is, only $d$ out of $|\mathcal{S}| |\mathcal{A}|$
  entries of $\mathbf{r}$ are allowed to deviate from its nominal
  value, and each entry can deviate at most by $c$. Note that
  $\mathcal{U}$ is not a convex set. However, due to linearity of the
  objective function, we can optimize instead over the convex hull of
  $\mathcal{U}$, denoted by $\hat{\mathcal{U}}$, and the optimal
  solution to \eqref{equ.robustMDPinLP} over $\mathcal{U}$ and
  $\hat{\mathcal{U}}$ coincides. Moreover, the diameter of
  $\hat{\mathcal{U}}$ is bounded above by
  $2 \cdot\mbox{diameter}(\mathcal{U}) \leq 2 \sqrt{d} c$.
\end{remark}

\section{Computational Experiments}
\label{sec:comp-exper}

In this section we report computational experiments and while not a
focus of this paper, we demonstrate the real-world practicality of our
approach for Example~\ref{exa:robustRouting}. Following the approach
of Section~\ref{sec:robust-optim-with} we run two FPL learners where
the primal player is solving a min-cost flow problem over a graph,
whose polyhedral formulation is integral and the dual player is
playing the uncertainty.

Here we report results for three
representative instances, which where completed within a few minutes
of computational time. The smaller one has \(n=50\) nodes, the second
one is a very dense graph instance with \(n = 100\) nodes, and the
larger one has \(n=1000\) nodes and is relatively sparse. We plot
\(\max_{\mbu'\in \mathcal{U}} f(\overline{\mbx}, \mbu')
-\min_{\mbx^*\in \mathcal{X}} \max_{\mbu\in \mathcal{U}} f(\mbx^*,
\mbu)\) (which we refer to as \lq{}regret\rq{} slightly abusing
notions) against iterations \(t\) in
Figures~\ref{fig:experiments} and~\ref{fig:experiments2}. In all cases
our algorithm based on imaginary play converges very fast (note that
the figures report in log scale) and in fact much faster as predicted by
the conservative worst-case bound given in
Example~\ref{exa:robustRouting}.
 \begin{figure}[htb!]
  \centering
  \includegraphics[width=0.49\linewidth]{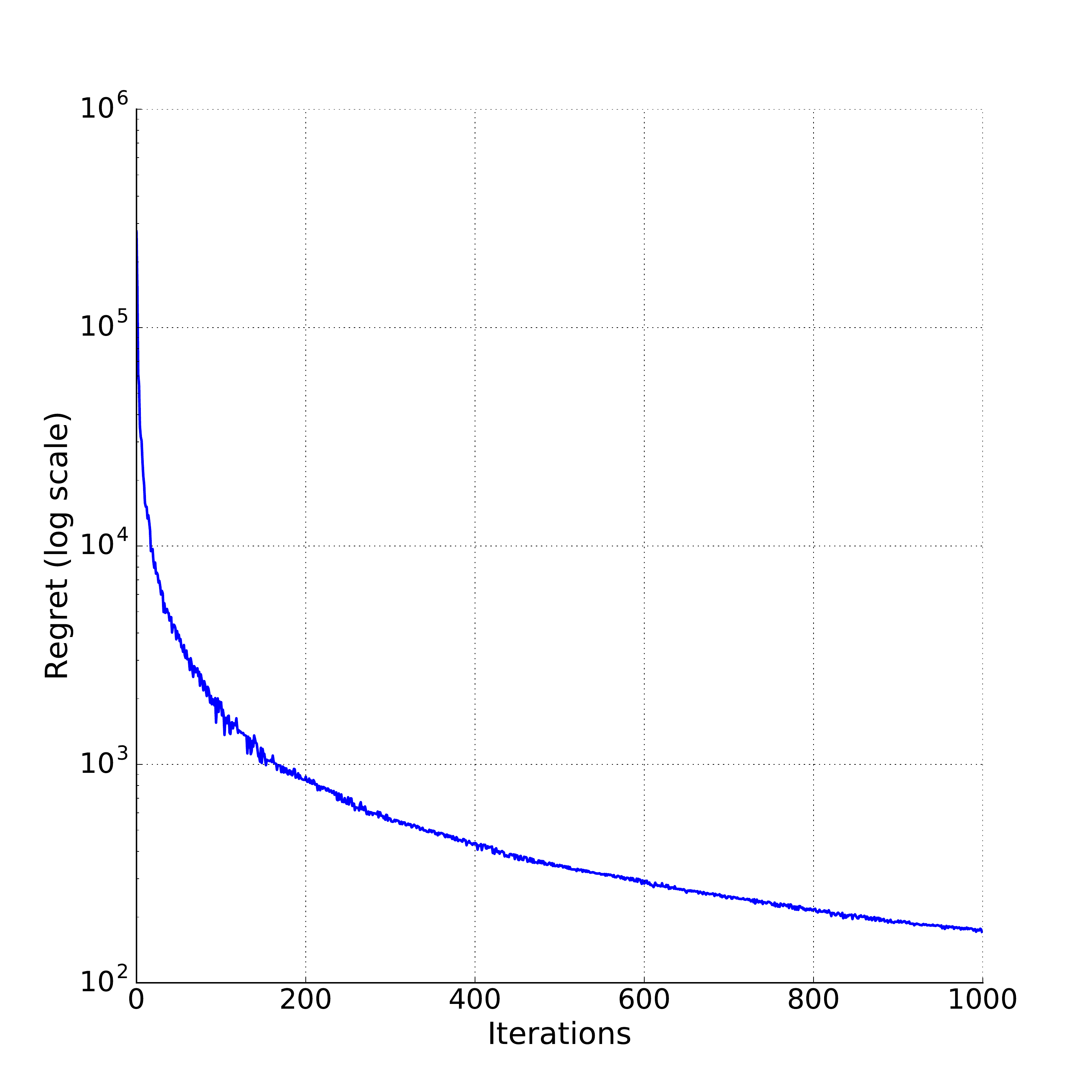}
  \includegraphics[width=0.49\linewidth]{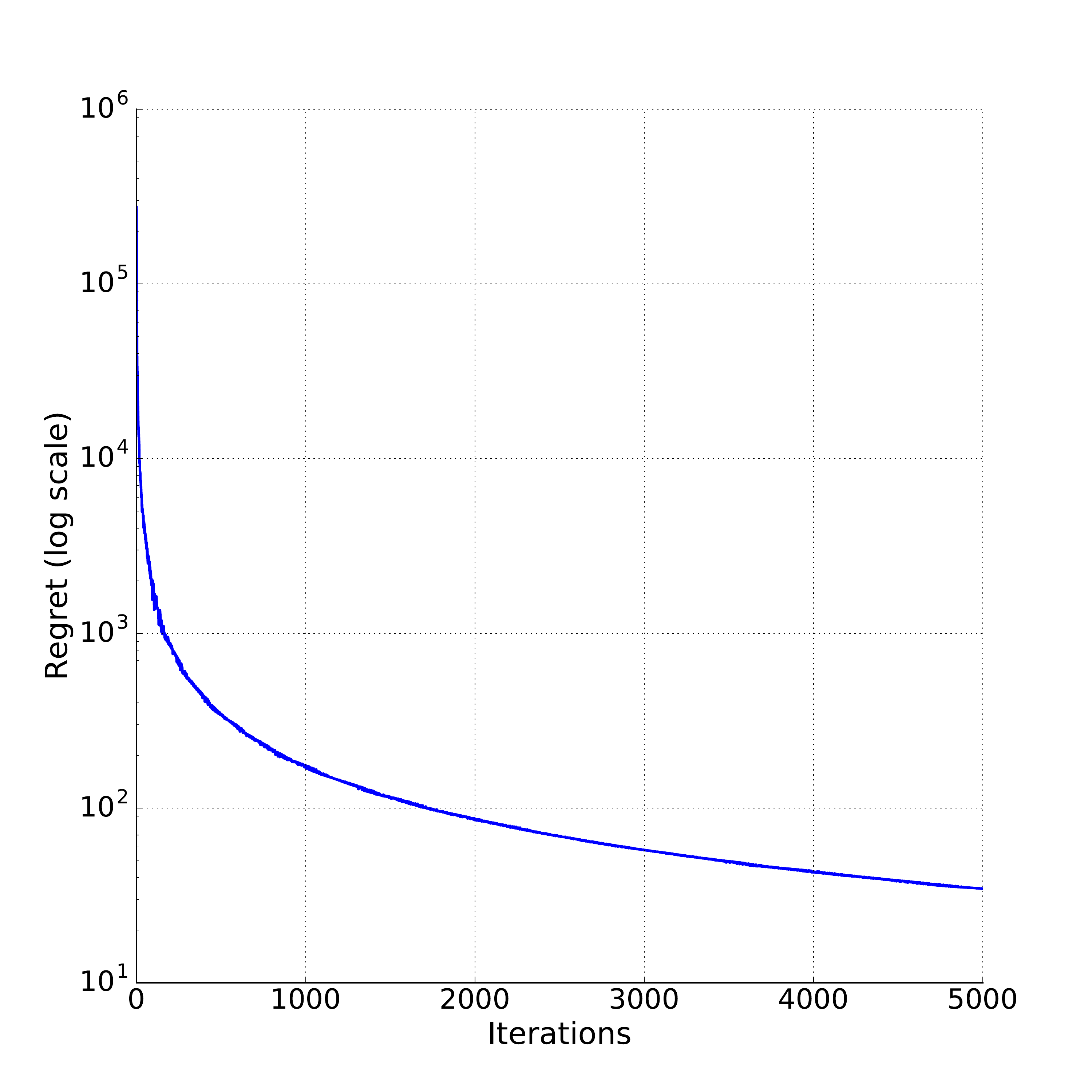}
  \caption{\label{fig:experiments} Regret (log scale)
    vs.~Iterations. Left: Robust min cost flow with \(n=50\) nodes run for
    \(1000\) iterations. Right: Same instance run for \(5000\);
    approaching zero regret. }
\end{figure}
\begin{figure}[htb!]
  \centering
  \includegraphics[width=0.49\linewidth]{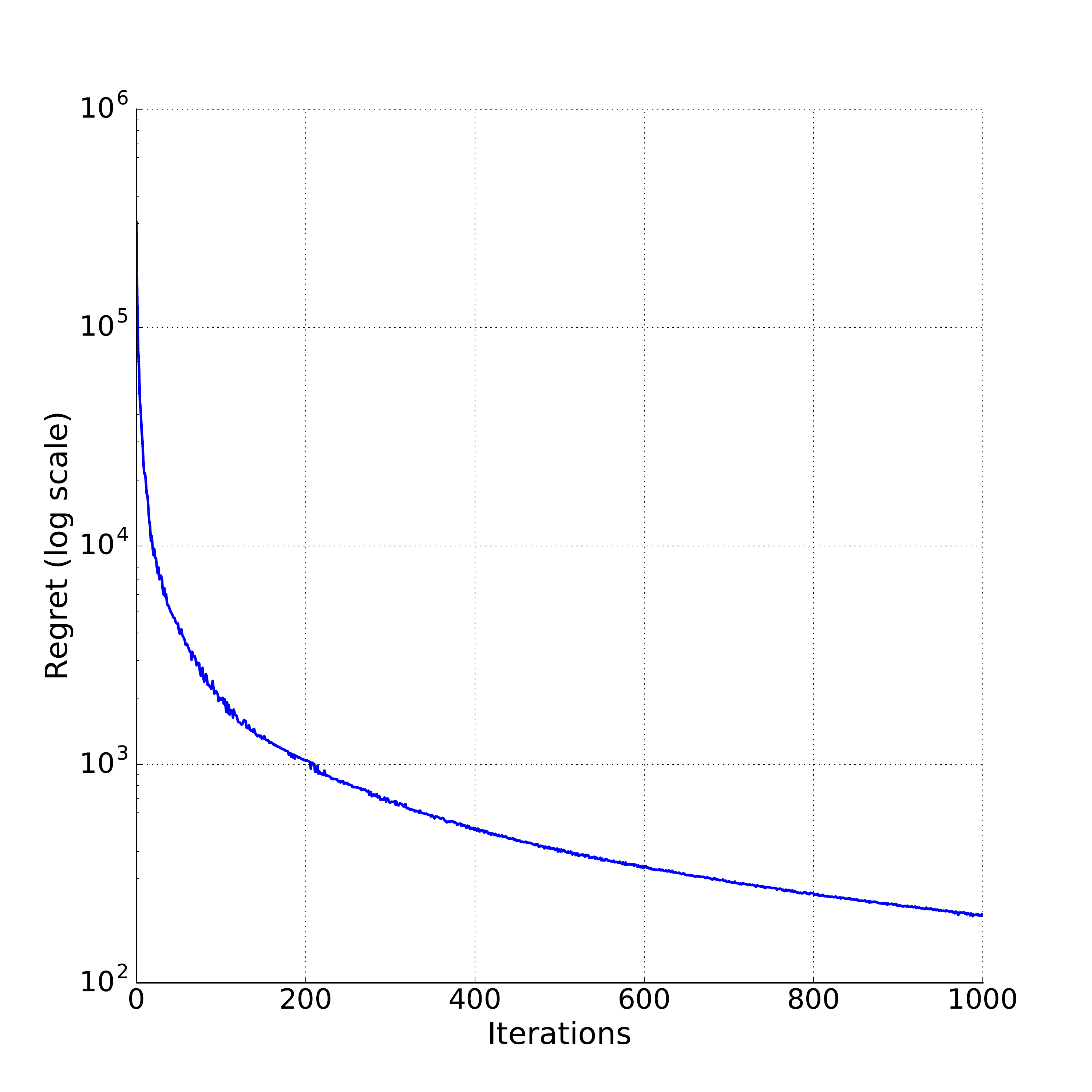}
  \includegraphics[width=0.49\linewidth]{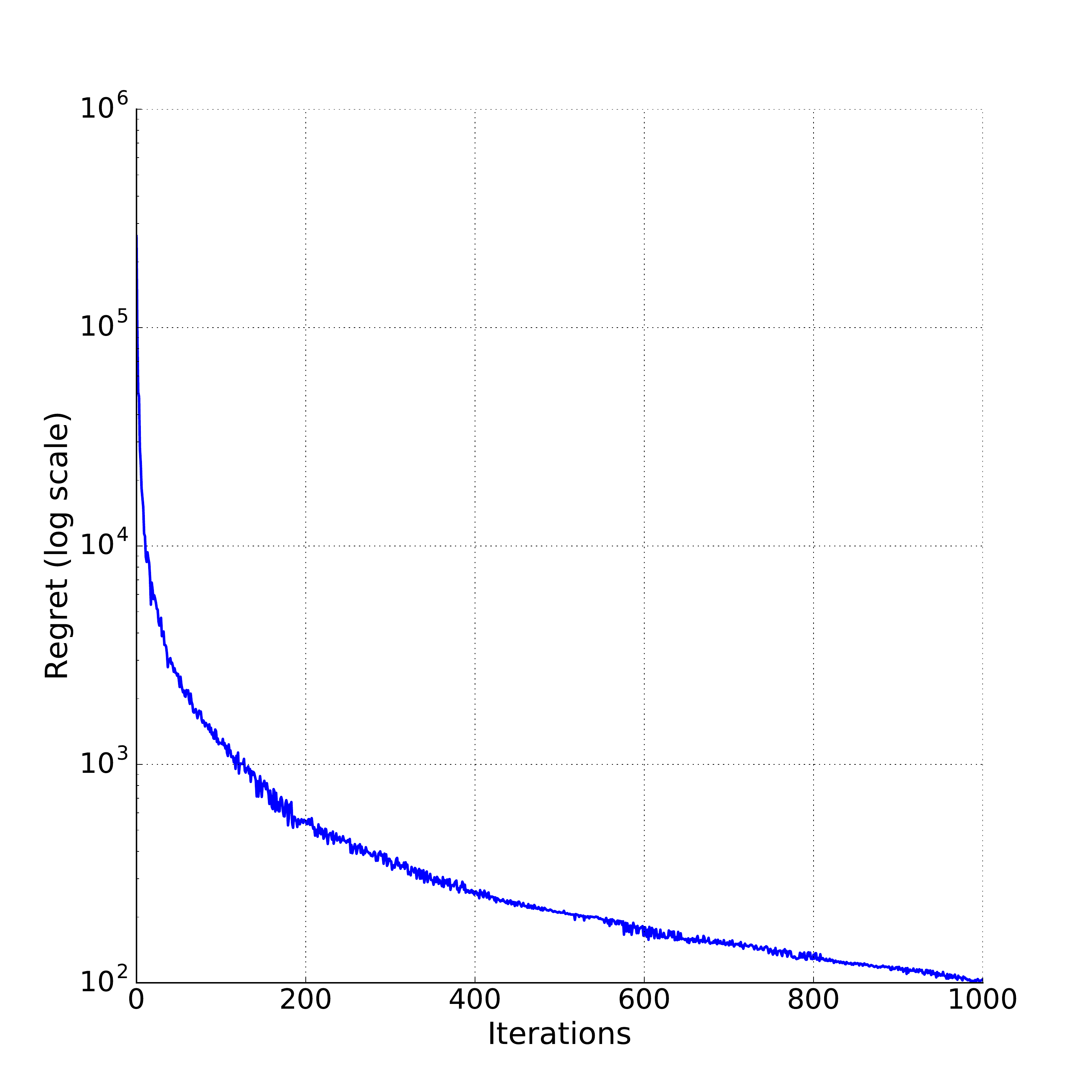}
  \caption{\label{fig:experiments2} Regret 
    (log scale) vs.~Iterations for larger instance. Left: Very dense
    graph instance on \(n=100\) nodes. Right: Instance with \(n=1000\)
    nodes run for \(1000\) iterations.}
\end{figure}

\subsection{Implementation details}
\label{sm:compRes}

We now provide implementation details for the computational experiments
reported above for reproducibility of results. The implementation is in
\texttt{python 3.5} using the network simplex algorithm from the
\texttt{networkx} library. All computational experiment were
completed in a few minutes on a standard laptop with a 2.7 GHz Intel
Core i5 processor; time per iteration is between \(0.2s\) and \(1s\)
depending on the size of the graph (note that the network simplex in
\texttt{networkx} is a pure python implementation).

The graph instances were generated using \texttt{networkx}'s
\texttt{fast\_gnp\_random\_graph} function with edge probabilities \(p\)
varying between \(1\%\) and \(30\%\). Demand was set to be \(1\) unit to
be sent from node \(0\) to node \(1\), without loss of generality due
to symmetry of the random graph generation. Both nominal edge
capacities as well as costs where set to uniform random values between
\(0\) and \(1\) inducing non-trivial min cost flows in the nominal
problem for the uncertainty realizations. 

Following the approach of Section~\ref{sec:robust-optim-with} we run
two FPL learners where the primal player is solving a min-cost flow
problem over the graph, whose polyhedral formulation is integral and the
dual player is playing uncertainty in terms of congestion, with high
cost \(M\) corresponding to edge removal. The value of \(M\) is
computed dynamically for each instance. The dual player plays congestion
uncertainty either (a) from an ellipsoidal uncertainty set initialized with a
random matrix (here the worst-case uncertainty can be computed
directly via algebraic manipulations), or (b) from a budget constraint
set of the form \(\{x \geq 0 \mid \sum_i {x_i} \leq K\}\) (here the
worst-case uncertainty can be obtained via sorting); the results and timing
are very similar for both uncertainty types. Our FPL implementations
are those of \cite{kalai2005efficient} as we are in the
full-information case.

In all figures we plot
\(\max_{\mbu'\in \mathcal{U}} f(\overline{\mbx}, \mbu') -
\min_{\mbx^*\in \mathcal{X}} \max_{\mbu\in \mathcal{U}} f(\mbx^*,
\mbu)\) in log-scale (which we refer to as \lq{}regret\rq{} slightly abusing
notions) against iterations \(t\). The right hand side value
\(\min_{\mbx^*\in \mathcal{X}} \max_{\mbu\in \mathcal{U}} f(\mbx^*,
\mbu)\) has been computed a priori to enable plotting of the regret;
in actual implementations this is of course not needed as the number of
iterations imply a strong bound on the quality of the solution but
here we wanted to make convergences to zero regret explicit. 

Generally it can be observed that in all cases the algorithm based on
imaginary play converges very fast and in fact much faster as predicted by the conservative
worst-case bound given in Example~\ref{exa:robustRouting}

We provide additional computations in Figures~\ref{fig:experiments3}.

\begin{figure}
  \centering
  \includegraphics[width=0.49\linewidth]{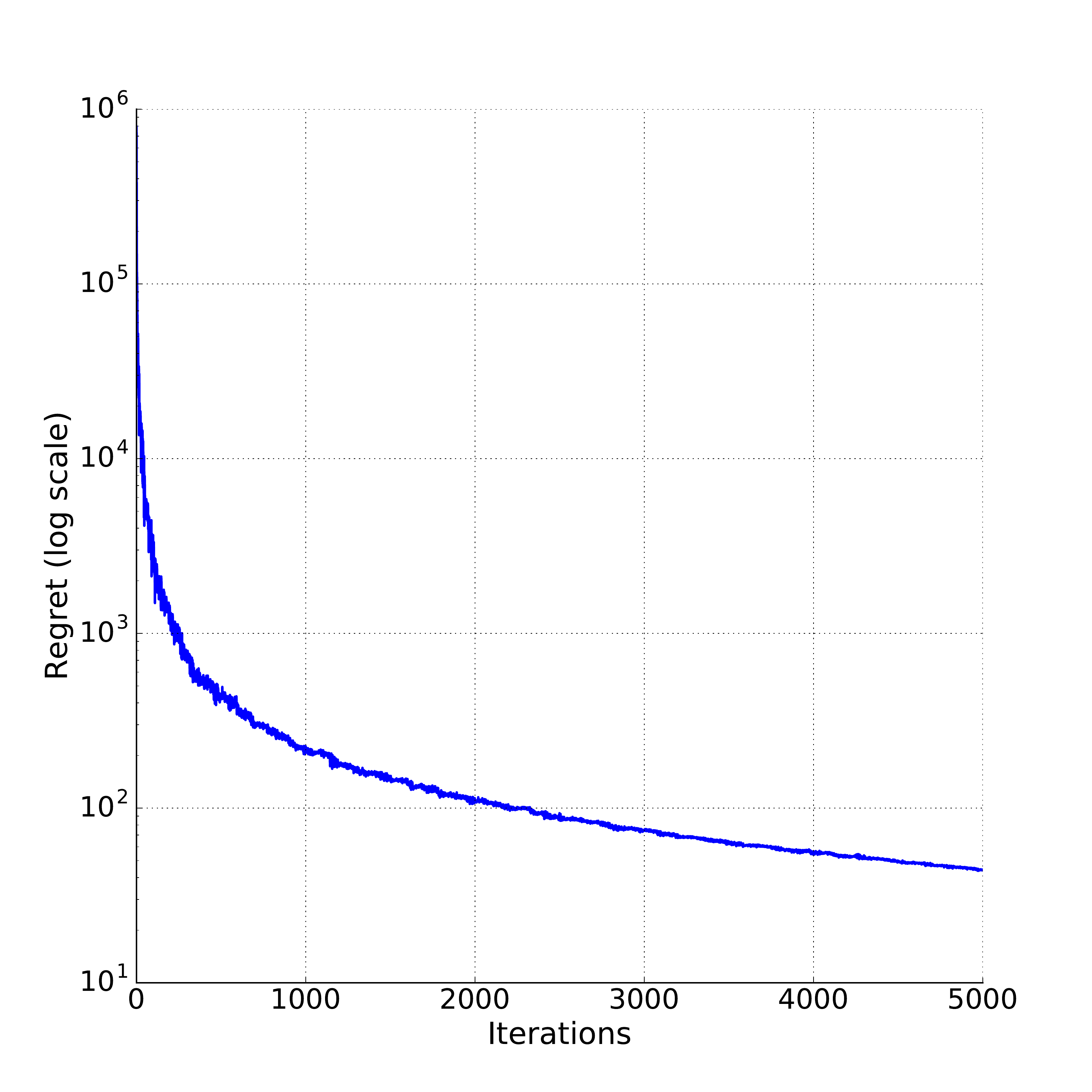}
  \includegraphics[width=0.49\linewidth]{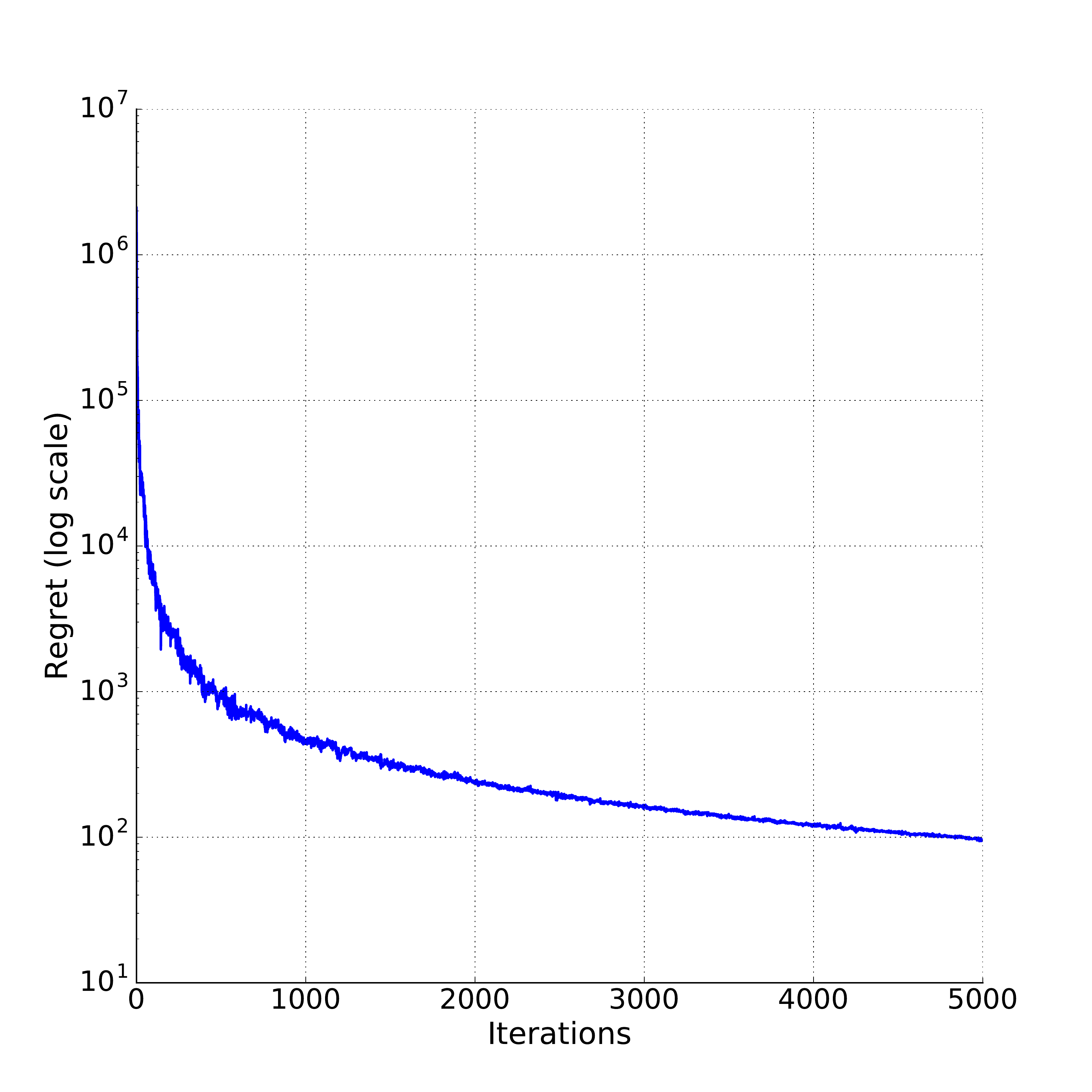} \\
  \includegraphics[width=0.49\linewidth]{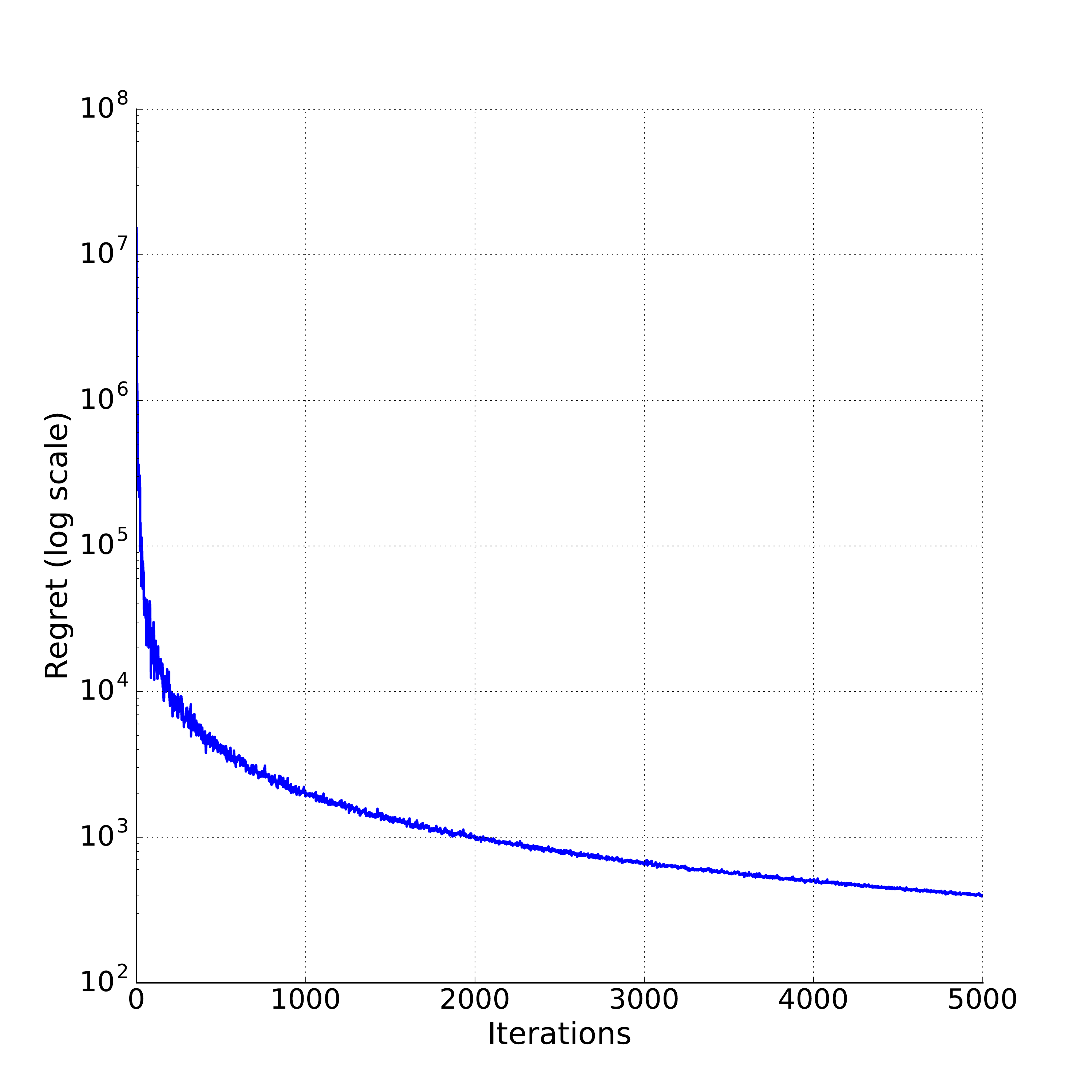}
  \includegraphics[width=0.49\linewidth]{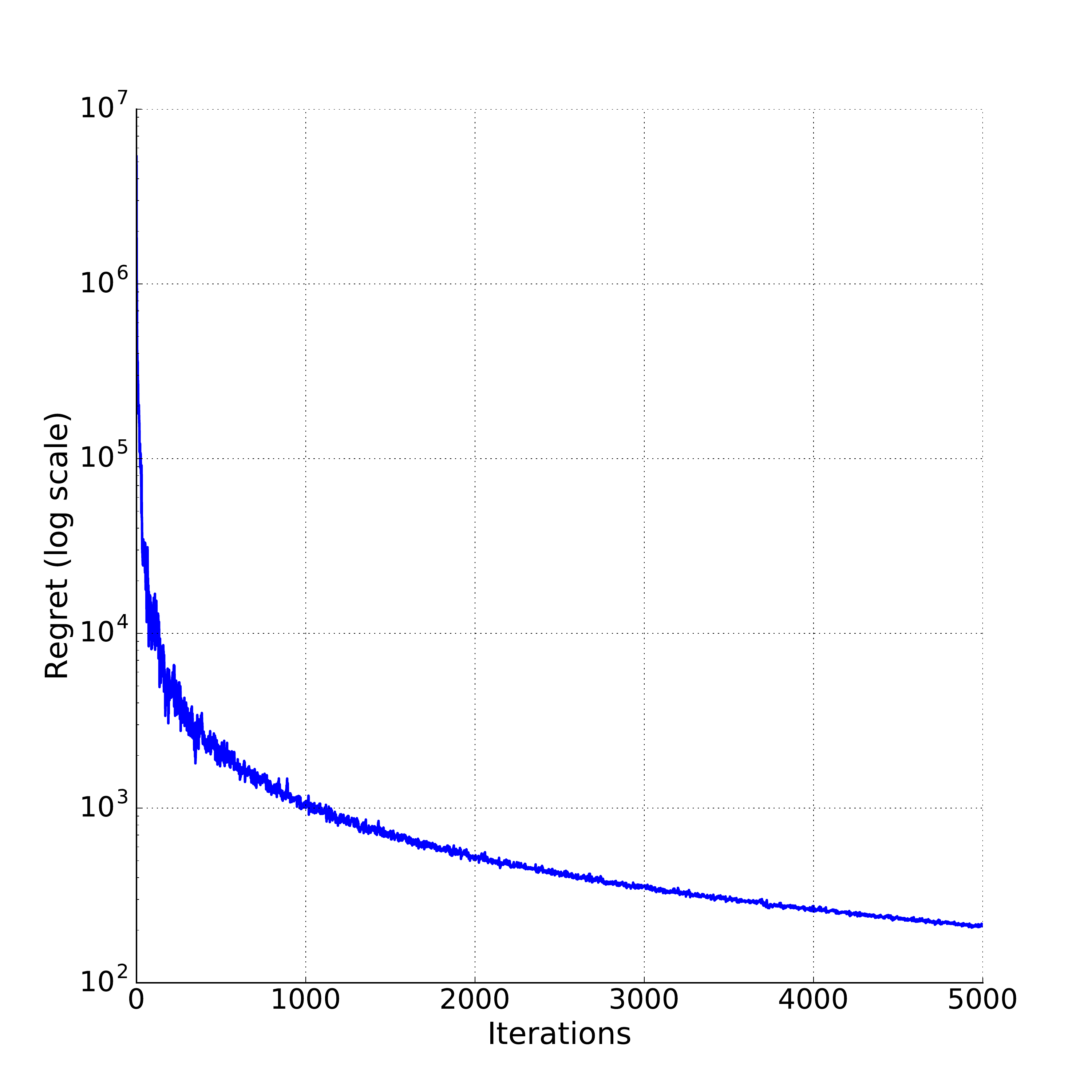} 
  \caption{\label{fig:experiments3} Regret (log scale)
    vs.~Iterations. Left/Top: Robust min cost flow with \(n=400\)
    nodes and \(p = 0.1\) run for
    \(5000\) iterations with \(\ell_2\)-uncertainty set with radius
    \(5\). Right/Top: Robust min cost flow with \(n=200\) and \(p = 0.3\) nodes run for
    \(5000\) iterations with \(\ell_2\)-uncertainty set with radius
    \(20\). Left/Bottom: Robust min cost flow with \(n=200\) and \(p = 0.3\) nodes run for
    \(5000\) iterations with \(\ell_2\)-uncertainty set with radius
    \(100\). Right/Bottom: Robust min cost flow with \(n=200\) and \(p = 0.3\) nodes run for
    \(5000\) iterations with budgeted uncertainty with
    \(K = 50\).}
\end{figure}

\end{document}